\documentclass[11pt,letterpaper]{berkeley}

\usepackage[all]{hypcap}

\usepackage[authoryear, round]{natbib}
\usepackage{hyperref}[citecolor=lightblue]

\hypersetup{
    colorlinks = true,
    citecolor = {magenta},
}

\usepackage{microtype}
\usepackage{graphicx}
\usepackage{subfigure}
\usepackage{booktabs} %
\usepackage{float}

\usepackage{amsmath}
\usepackage{amssymb}
\usepackage{mathtools}
\usepackage{amsthm}
\usepackage{mathrsfs}
\usepackage{nicefrac}
\usepackage{dsfont}
\usepackage{enumitem}
\usepackage{float}

\usepackage{xspace}
\usepackage[capitalize,noabbrev]{cleveref}
\usepackage{subcaption}
\usepackage{wrapfig}
\usepackage{lipsum}
\usepackage{listings}

\usepackage{amsmath}
\usepackage{amssymb}
\usepackage{mathtools}
\usepackage{amsthm}
\usepackage{bbm}

\usepackage{algpseudocode}
\usepackage{setspace}

\usepackage{color}
\definecolor{deepblue}{rgb}{0,0,0.5}
\definecolor{deepred}{rgb}{0.6,0,0}
\definecolor{deepgreen}{rgb}{0,0.5,0}

\newcommand\pythonstyle{\lstset{
basicstyle=\ttfamily\footnotesize,
language=Python,
morekeywords={self, clip, exp, mse_loss, uniform_sample, concatenate, logsumexp},              %
keywordstyle=\color{deepblue},
emph={MyClass,__init__},          %
emphstyle=\color{deepred},    %
stringstyle=\color{deepgreen},
frame=single,                         %
showstringspaces=false
}}

\lstnewenvironment{python}[1][]
{
\pythonstyle
\lstset{#1}
}
{}

\newcommand\pythoninline[1]{{\pythonstyle\lstinline!#1!}}

\newcommand{\y}{\mathbf{y}}

\newcommand{\by}{\mathbf{y}}
\newcommand{\bx}{\mathbf{x}}
\newcommand{\bz}{\mathbf{z}}

\newcommand{\ba}{\mathbf{a}}

\definecolor{blanchedalmond}{rgb}{1.0, 0.92, 0.8}
\definecolor{carmine}{rgb}{0.59, 0.0, 0.09}
\definecolor{lightblue}{rgb}{0.22,0.45,0.70}%

\newcommand{\scenario}[1]{\textbf{\textcolor{lightblue}{[C#1]}}}

\newtheorem{theorem}{Theorem}[section]

\newtheorem{lemma}[theorem]{Lemma}

\renewcommand{\mathbf}{\boldsymbol}

\makeatletter
\def\Ddots{\mathinner{\mkern1mu\raise\p@
\vbox{\kern7\p@\hbox{.}}\mkern2mu
\raise4\p@\hbox{.}\mkern2mu\raise7\p@\hbox{.}\mkern1mu}}
\makeatother

\newcommand{\piref}{\pi_\mr{ref}}

\newcommand{\mr}{\mathrm}

\numberwithin{equation}{section}

\definecolor{amaranth}{rgb}{0.9, 0.17, 0.31}
\definecolor{antiquebrass}{rgb}{0.8, 0.58, 0.46}
\definecolor{antiquefuchsia}{rgb}{0.57, 0.36, 0.51}
\definecolor{chromeyellow}{rgb}{0.31, 0.47, 0.26}

\DeclareMathOperator*{\argmax}{arg\,max}

\makeatletter
\def\mathcolor#1#{\@mathcolor{#1}}
\def\@mathcolor#1#2#3{%
  \protect\leavevmode
  \begingroup
    \color#1{#2}#3%
  \endgroup
}
\makeatother

\usepackage[textsize=tiny]{todonotes}
\usepackage{algorithm}
\usepackage{amssymb}

\usepackage[skins,theorems]{tcolorbox}

\tcbset{
  aibox/.style={
    width=474.18663pt,
    top=7pt,
    bottom=5pt,
    colback=blue!6!white,
    colframe=black,
    colbacktitle=black,
    enhanced,
    center,
    attach boxed title to top left={yshift=-0.1in,xshift=0.15in},
    boxed title style={boxrule=0pt,colframe=white,},
  }
}
\newtcolorbox{AIbox}[2][]{aibox,title=#2,#1}

\Crefformat{equation}{#2Eq.\;(#1)#3}

\Crefformat{figure}{#2Figure #1#3}
\Crefformat{assumption}{#2Assumption #1#3}
\Crefname{assumption}{Assumption}{Assumptions}

\usepackage{crossreftools}
\usepackage{dsfont}
\usepackage{nicefrac}
\usepackage{inconsolata}

\title{{Preference Fine-Tuning of LLMs Should Leverage Suboptimal, On-Policy Data}}

\reportnumber{} %

\author[1*]{Fahim Tajwar}
\author[2*]{Anikait Singh}
\author[2]{Archit Sharma}
\author[2]{Rafael Rafailov}
\author[1]{Jeff Schneider}
\author[4]{Tengyang Xie}
\author[2]{Stefano Ermon}
\author[2]{Chelsea Finn}
\author[3]{Aviral Kumar}

\affil[*]{Equal contributions (ordered via coin-flip)}
\affil[1]{Carnegie Mellon University}
\affil[2]{Stanford University}
\affil[3]{Google DeepMind}
\affil[4]{UW-Madison}

\correspondingauthor{anikait@stanford.edu, ftajwar@cs.cmu.edu. \textbf{Project Website:} \url{https://understanding-rlhf.github.io/}}

\begin{abstract}
\vspace{-0.3cm}
Learning from preference labels plays a crucial role in fine-tuning large language models. There are several distinct approaches for preference fine-tuning, including supervised learning, on-policy reinforcement learning (RL), and contrastive learning. 
Different methods come with different implementation tradeoffs and performance differences, and existing empirical findings present different conclusions, for instance, some results show that online RL is quite important to attain good fine-tuning results, while others find (offline) contrastive or even purely supervised methods sufficient. This raises a natural question: \emph{\textbf{what kind of approaches are important for fine-tuning with preference data and why?}} In this paper, we answer this question by performing a rigorous analysis of a number of fine-tuning techniques on didactic and full-scale LLM problems. 
Our main finding is that, in general, approaches that use on-policy sampling or attempt to push down the likelihood on certain responses (i.e., employ a ``negative gradient'') outperform offline and maximum likelihood objectives. We conceptualize our insights and unify methods that use on-policy sampling or negative gradient under a notion of mode-seeking objectives for categorical distributions. Mode-seeking objectives are able to alter probability mass on specific bins of a categorical distribution at a fast rate compared to maximum likelihood, allowing them to relocate masses across bins more effectively. Our analysis prescribes actionable insights for preference fine-tuning of LLMs and informs how data should be collected for maximal improvement.
\end{abstract}

\begin{document}

\maketitle

\vspace{-0.5cm}
\section{Introduction}
\vspace{-0.2cm}

Pre-training endows a large language model (LLM) with knowledge about the world. Yet, it does not provide a lever to control responses from these models, especially when we want these solutions to optimize some task-dependent success criteria (e.g., align with human preferences, optimize correctness or compactness). To align LLMs with downstream success criteria, they are then fine-tuned with downstream objectives after pre-training. In this paper, we focus on fine-tuning problems that aim to optimize for binary preferences (from humans or other AI models). A plethora of methods have been proposed for this sort of fine-tuning, including supervised learning on filtered responses~\citep{gulcehre2023reinforced}, contrastive training~\citep{rafailov2023direct}, and on-policy reinforcement learning (RL)~\citep{ouyang2022training} on a reward function extracted from human preferences. 

In theory, while all of these methods aim to discover identical optimal policies, achieving this in practice would require full data coverage and infinite computation. These requirements are not met in practice, and hence, the choice of the loss function and the optimization procedure affects performance. However, a lack of a clear understanding of different approaches, coupled with different tradeoffs in implementation, has resulted in substantial confusion: practitioners are unsure as to: \textbf{(1)} whether RL~\citep{ouyang2022training} is required at all, or contrastive approaches~\citep{rafailov2023direct,2023arXiv231012036G}, supervised fine-tuning are good enough; and \textbf{(2)} whether preference data should be collected with models in the loop (i.e., in an ``on-policy'' fashion) or not.

Our goal is to provide clarity on these questions by performing a rigorous study to understand the behavior of existing methods when optimizing for preferences. Our study operates under assumptions typical in preference fine-tuning, including the existence of an underlying ground truth reward function that explains the preference data. We study methods that train an LLM policy to optimize a surrogate loss given by the expected reward under a model of the reward function (learned from preference data) penalized by the KL-divergence between the policy and a reference policy. 

\begin{figure}[t!]
    \centering
    \vspace{-0.4cm}
    \includegraphics[width=0.9\textwidth]{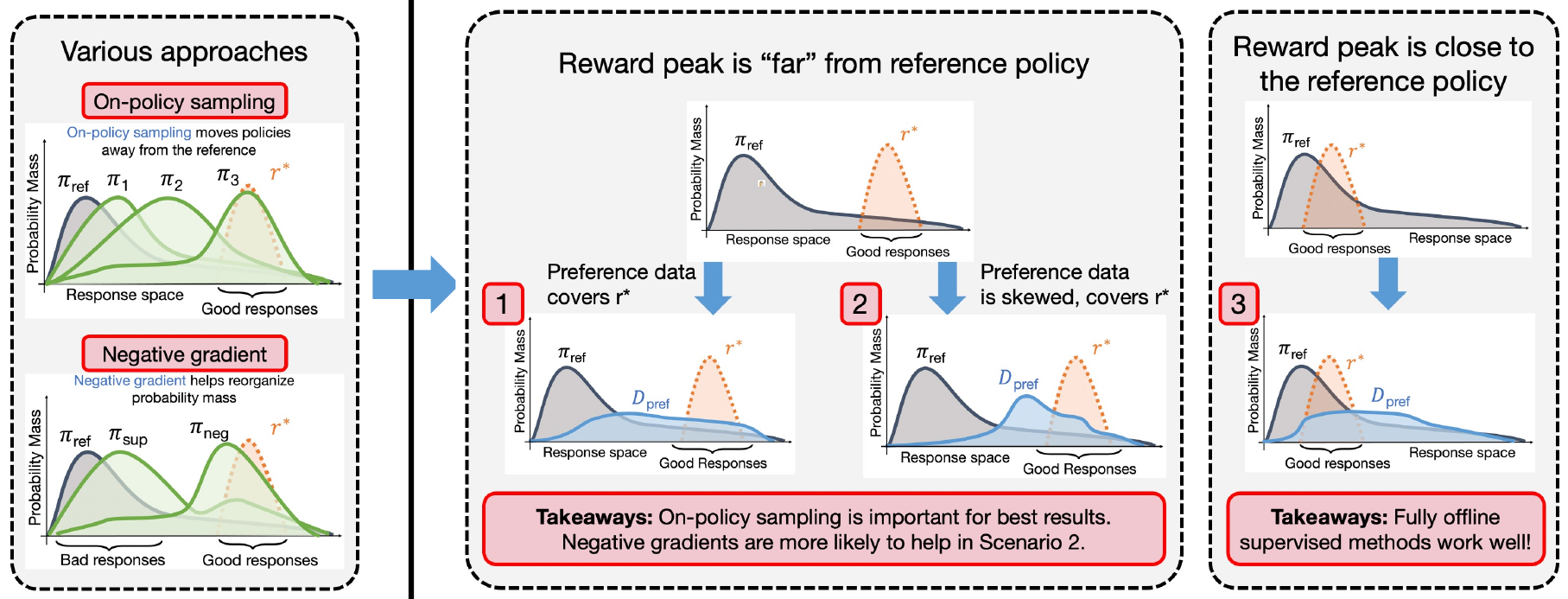}
    \vspace{-0.25cm}
    \caption{\footnotesize{\textbf{Left: an illustration of various fine-tuning techniques.} On-policy sampling gradually shifts policy mass from $\piref$ to $\pi_i$, moving it towards the peak in the reward function indicated by $r^*$. Offline methods that employ a negative gradient push down the likelihood of bad responses under the learned policy, resulting in a larger deviation of $\pi_\text{neg}$ compared to policies that only maximize some sort of likelihood, $\pi_\text{sup}$. \textbf{Right: our key takeaways for practitioners:} When the peak of the reward function lies in the less likely regions of $\piref$, on-policy sampling is generally beneficial. In conjunction, an explicit negative gradient approach (e.g., via RL objectives or via contrastive objectives) is beneficial when the preference data is skewed away from $\piref$ (case 2). When $r^*$ already lies in the high likelihood regions of $\piref$, offline supervised methods can work well. No on-policy sampling or negative gradients may be needed.}}
    \label{fig:teaser}
    \vspace{-0.35cm}
\end{figure}

To answer the above questions, we develop an analysis framework consisting of didactic bandit problems, synthetic LLM problems, and full-scale LLM problems, constructed out of AlpacaFarm~\citep{dubois2024alpacafarm} and UltraFeedback~\citep{2023arXiv231001377C}. We then study behaviors of different methods given coverage conditions and geometric relationships in the problem. \emph{\textbf{Our main observation}} {is that algorithms that use on-policy RL in a reward model or attempt to push-down likelihood on certain responses, i.e., utilize a negative gradient term as in contrastive objectives tend to outperform other offline supervised objectives} with no on-policy sampling or negative gradient. \textbf{This is surprising} because both on-policy and offline methods still utilize the same data for learning. We also find that \textbf{using on-policy sampling and negative gradients are especially important} when high-reward responses appear in less-likely regions of the reference policy distribution, and provide benefits complementary to each other. In particular, we find that supervised objectives such as Pref-FT  and Binary Feed-ME~\citep{dubois2024alpacafarm} are not able to effectively move probability mass from low reward responses to high-reward responses. Sampling on-policy responses for training, contrastive learning, or employing both on-policy sampling and contrastive training can accomplish this. 

\textbf{We theoretically show} that approaches that use on-policy RL or certain variants of contrastive training exhibit ``mode-seeking'' behavior, resulting in faster accumulation of probability mass on a subset of high-reward responses during learning. This behavior is in contrast to ``mode-covering'' supervised objectives that attempt to increase likelihood on all high-reward responses, and as a result, are unable to efficiently increase probability mass enough on one subset of high-reward responses. We then compare the behavior of a representative mode-seeking objective, the reverse KL-divergence, with the mode-covering forward KL-divergence to formalize this behavior for categorical distributions. Conceptually, this ability to commit to a certain subset of high-reward responses enables algorithms with on-policy sampling (and optionally, a negative gradient) to perform better than likelihood.

\textbf{Our work presents several actionable takeaways for practitioners.} First, we tie the performance of various methods to geometric conditions on the problem, which can inform practitioners which approach to use. Second, we observe a tradeoff between drawing more on-policy samples and performing more gradient steps with a different policy training objective. Understanding this tradeoff is useful for practitioners since on-policy sampling and training present different computational tradeoffs. Finally, since the performance of fine-tuning is tied to the data composition, we study the effect of conditions on the coverage of the preference data, which could inform data collection. 

\vspace{-0.3cm}
\section{Related Work}
\vspace{-0.2cm}

A dominant recipe for fine-tuning LLMs is to run supervised next token prediction (``supervised fine-tuning'') on a dataset of high-quality responses to obtain a good policy initialization. This is followed by fine-tuning on a dataset of human preferences~\citep{casper2023open,ouyang2022training}. This fine-tuning can use on-policy RL methods such as REINFORCE~\citep{NIPS1999_464d828b} or  PPO~\citep{2017arXiv170706347S} to maximize the predictions of a reward model obtained from the preference data, regularized with a KL constraint. Another approach~\citep{dubois2024alpacafarm} performs supervised fine-tuning on the filtered set of preferred completions in the preference dataset. A different family of methods runs supervised learning on preferred responses iteratively such as ReST~\citep{gulcehre2023reinforced}, RWR~\citep{2023arXiv230812050H}, and SuperHF~\citep{superhf}. Alternatively, methods such as DPO~\citep{rafailov2023direct}, IPO~\citep{2023arXiv231012036G}, SLiC-HF~\citep{2023arXiv230510425Z}, and KTO~\citep{HALOs2024} learn directly from human preferences, with no explicit reward model. Concurrent work also runs DPO iteratively~\citep{2024arXiv240110020Y,2024arXiv240101335C}. These methods come with different tradeoffs necessitating a study to understand their behaviors.

\textbf{Prior analysis work.} To understand the effect of preference fine-tuning, prior work attempts to uncover its effect on network parameters for a certain set of tasks~\citep{2023arXiv231112786J,2024arXiv240101967L}. Our analysis is complementary in that it studies conditions when different algorithms perform well, and is applicable to any downstream task. \citet{2023arXiv231006452K} study the contribution of RL fine-tuning on generalization to out-of-distribution prompts but this is complementary to our approach. \citet{2022arXiv221010760G, 2023arXiv231002743C, 2023arXiv231209244E} study reward over-optimization to better build reward models, which is complementary to the behavior of the policy optimization approach. 
\citet{agarwal2023gkd} develop a recipe that uses the mode-seeking KL divergence for knowledge distillation: this prior work is largely centered in the problem setting of distillation and does not study the optimization behavior of RL, contrastive, or supervised objectives.
Perhaps closely related to our work is \citet{singhal2023long}, which investigates the interplay between PPO and the composition of preference data, but this analysis is largely concentrated on studying the length bias of RL fine-tuning rather than developing insights into the behavior of fine-tuning algorithms. We do design didactic examples that use rewards dependent on length, but this is solely for analysis.

Concurrently, \citet{ahmadian2024back} show that REINFORCE may simply be enough for preference fine-tuning of LLMs and complex policy optimization methods such as PPO may not be needed. Our conclusions are mostly complementary, though we do observe that PPO is more robust to sample reuse than REINFORCE. Concurrently, \citet{sharma2024critical} compares contrastive and supervised fine-tuning on LLM-generated data, but this work does not study the role of coverage or geometric conditions. Nevertheless their conclusions that various approaches perform similarly when the peak in the reward function (i.e., oracle AI preferences) aligns with the likely regions in the data (i.e., responses generated from the same AI model), thus providing evidence to support our findings.

\vspace{-0.2cm}
\section{Characterizing And Unifying Preference Fine-Tuning Methods}
\label{sec:background}
\vspace{-0.15cm}

Typical preference fine-tuning methods use a variety of objectives including RL, maximum likelihood, and contrastive learning. While the huge number of fine-tuning methods inhibits us from empirically analyzing each of them, in this section we characterize several existing methods into different families and subsequently study a representative member from each family. 

\vspace{-0.2cm}
\subsection{Preliminaries and Notation}
\label{subsec:prelims}
\vspace{-0.2cm}

Typically, before training on preference data, a pre-trained model is fine-tuned on high-quality data from the task of interest via supervised fine-tuning (SFT), to obtain a ``reference'' model $\piref$. Then, to fine-tune $\piref$ with human preferences, usually a preference dataset $\mathcal{D}_{\text{pref}} = \{\bx^{(i)}, \by_w^{(i)}, \by_l^{(i)}\}$ is collected, where $\bx^{(i)}$ denotes a prompt and $\by_w^{(i)}, \by_l^{(i)}$ denote preferred and dispreferred responses, obtained typically by sampling from $\piref$. Given a preference dataset, most fine-tuning pipelines assume the existence of an underlying reward function $r^*(\bx, \cdot)$. One popular framework for this is the Bradley-Terry (BT) model~\citep{bradleyterry1952preferences}, assuming that human preferences can be written as:
\begin{align} 
\label{eq:bradley_terry}
    \textstyle p^*(\by_1 \succ \by_2 | \bx) = \frac{e^{r^*(\bx, \by_1)}}{e^{r^*(\bx, \by_1)} + e^{r^*(\bx, \by_2)}}
\end{align}
Given this reward function $r^*$, preference fine-tuning aims to find the optimum of the reward $r^*$. While the ultimate goal of preference fine-tuning is to find the \emph{unconstrained} optimum of the reward function, in practice, we often replace the reward function with a reward model. Since the reward model is erroneous, we apply  KL-constraint to prevent exploitation in the reward model. To align our results with typical preference fine-tuning procedures, we will consider such a KL-constrained reward optimization as our fine-tuning goal:
\begin{align} 
\label{eq:rlhf_objective}
    \!\!\!\!\max_{\pi_{\theta}}~ & \mathbb{E}_{\bx \sim \mathcal{D}_{\text{pref}}, \by \sim \pi_{\theta}(\cdot|x)} [r^*(\bx, \by)] - \beta \mathbb{D}_{\text{KL}}[\pi_{\theta}(\cdot|\bx) || \pi_{\text{ref}}(\cdot|\bx)] ~~ \text{\textcolor{red}{\textbf{(Surrogate fine-tuning goal)}}}
\end{align}
The regularizer, weighted by $\beta$, controls the deviation of $\pi$ from $\piref$ under the reverse KL divergence. 

\textbf{Reward model training.} In order to fine-tune an LLM policy $\pi_\theta(\by|\bx)$, \cref{eq:bradley_terry} provides a convenient way to learn a reward model either explicitly (i.e., by fitting a parametric reward model $r_\phi(\bx, \by)$) or implicitly (i.e., via direct preference optimization (DPO)~\cite{rafailov2023direct} or IPO~\citep{2023arXiv231012036G}, that re-purposes the log-likelihood $\log \pi_\theta(\by|\bx)$ of the policy to represent the reward $r_\theta(\bx, \by)$). Explicit reward models are trained using the following classification objective:
\begin{align} \label{eq:reward_learning}
\!\!\!\max_\phi~ \mathbb{E}_{(\bx, \by_w, \by_l) \sim \mathcal{D}_{\text{pref}}} \left[\log \sigma \left(r_\phi(\bx, \by_w) - r_\phi(\bx, \by_l) \right) \right],
\end{align}
where $\sigma$ is the logistic function. Contrastive learning objectives~\citep{rafailov2023direct,2023arXiv231012036G} on the other hand repurposes $\log \pi_\theta(\by|\bx)$ as the implicit reward $r_\theta(\bx, \by)$:
\begin{align}
\label{eq:contrastive_parameterization}
    r_\theta(\bx, \by) = \beta \left[ \log \pi_\theta(\by|\bx) - \log \piref(\by|\bx)\right].
\end{align}

\vspace{-0.2cm}
\subsection{Characterizing Fine-Tuning Methods}
\label{subsec:characterization}

With a reward model $r_\phi(\bx, \by)$, most fine-tuning approaches attempt to discover the policy $\pi_\theta(\by|\bx)$ which optimizes \cref{eq:rlhf_objective} by using $r_\phi$ as a surrogate for $r^*$. Since we cannot empirically investigate all of these methods, we group them into different categories (summary shown in \cref{table:characterization}). In particular, we are interested in whether these methods employ:

\vspace{-0.2cm}
\begin{enumerate}
\item \textbf{\emph{on-policy sampling}}: an explicit sampling of new responses from the policy (e.g., PPO, REINFORCE) or purely learning from offline data (e.g., RWR, DPO, IPO)
\item \textbf{\emph{on-policy sample reuse}:} for only those approaches that perform on-policy sampling, whether the approach makes more than one gradient update on a given prompt-response $(\bx, \by)$ pair (e.g., exactly 1 update for REINFORCE, $\geq 1$ for PPO, online RWR)
\item \textbf{\emph{negative gradient}:} whether the approach explicitly minimizes a loss that attempts to ``push-down'' likelihood on certain responses by multiplying the gradient of their likelihood with a negative coefficient (e.g., contrastive methods such as DPO; RL methods REINFORCE, PPO)
\end{enumerate}
\vspace{-0.3cm}

\textbf{On-policy RL} approaches such as PPO~\citep{schulman2017proximal} and REINFORCE~\citep{williams}
explicitly sample new responses from the current snapshot of the learned policy, $\by_i \sim \pi_\theta(\cdot|\bx_i)$, score them under the reward model, and perform a policy gradient update on parameters $\theta$, for example:
\begin{align}
    \label{eq:policy_grad}
    \theta' \leftarrow  \theta -  \eta \mathbb{E}_{\bx \sim \mathcal{D}_\text{pref}, \by \sim \pi_\theta(\cdot|\bx)} \left[\nabla_\theta \log \pi_\theta(\by|\bx) \cdot \bar{r}_\phi(\bx, \by)\right]~~~~~~~~~~ \text{(REINFORCE)},
\end{align}
is the gradient update employed by REINFORCE, where $\bar{r}_\phi(\bx, \by)$ corresponds to a normalized estimate of the reward model's predictions over a batch of samples drawn from the policy. As we discuss in more detail in \cref{subsection:reward_normalization}), using a normalized reward estimate instead of directly the raw reward value helps reduce the variance of the policy gradient estimate. High variance gradients slow down convergence and even sometimes lead to sub-optimal solutions in deep RL~\citep{mei2022role}.

Due to the use of normalized reward estimates, policy gradient approaches behave distinctly from maximum likelihood supervised learning: a policy gradient update also updates the parameters $\theta$ in a direction that attempts to push down likelihood $\log \pi_\theta(\by'|\bx)$ for samples $\by'$ on which normalized reward $\bar{r}_\phi(\bx, \by') < 0$. This means that on-policy RL also has a form of the \textbf{``negative gradient''}. 

\textbf{PPO differs from REINFORCE} because it employs \emph{sample reuse} in addition to on-policy sampling: unlike REINFORCE which only performs a single gradient update on a response sampled from the current policy, PPO can utilize a response for several policy updates. To prevent making updates on overly off-policy responses, there is a mechanism in place to filter responses by the magnitude of the importance ratio between the current policy $\pi_\theta(\by|\bx)$ and the data collection policy.

Finally, we also remark that while on-policy methods do generate new rollouts from the policy, these responses are still scored by a reward model (and not the ground truth reward function, i.e., humans). Since reward labels come from a reward model, on-policy preference fine-tuning approaches are instances of \textbf{offline model-based RL}~\citep{yu2021combo,yu2020mopo,kidambi2020morel} methods that run on-policy rollouts against a learned dynamics and reward model (due to the single step nature of preference fine-tuning, there is no dynamics model).

\begin{table}[h!]
\centering
\resizebox{0.95\textwidth}{!}{\begin{tabular}{c || c  c c} 
 \toprule
\textbf{Fine-Tuning Approach} & \textbf{On-Policy Sampling} & \textbf{Sample Reuse} & \textbf{Negative Gradient} \\
\midrule
PPO & $\checkmark$ & $\checkmark$ & $\checkmark$ \\
REINFORCE & $\checkmark$ & $\times$ & $\checkmark$ \\
\midrule
DPO, IPO, and variants & $\times$ & N/A & $\checkmark$ \\
\midrule
Pref-FT, Binary FeedMe & $\times$ & N/A & $\times$ \\
offline RWR, offline Best-of-N & $\times$ & N/A & $\times$ \\
\midrule
ReST, RWR, online Best-of-N &  $\checkmark$ & $\checkmark$ & $\times$ \\
\bottomrule
\end{tabular}}
\vspace{-0.25cm}
\caption{\footnotesize{\textbf{Grouping various fine-tuning methods} along the axes on-policy sampling, sample reuse, and negative gradient. Since offline methods do not collect on-policy data, the question of discarding or reusing on-policy samples is not applicable.}}
\label{table:characterization}
\end{table}

\textbf{On-policy supervised approaches} such as RAFT~\citep{dong2023raft}, ReST~\citep{gulcehre2023reinforced}, and SuperHF~\citep{superhf} iteratively minimize a weighted maximum likelihood loss inspired by \citet{peters2007reinforcement,korbak2022reinforcement}. For a given prompt $\bx_i$, these methods sample $N$ responses from the model: $\by^1_i, \cdots, \by^N_i \sim \pi_\theta(\cdots|\bx_i)$, then weight these responses by the exponentiated reward, $\exp(r_\phi(\bx_i, \by_i^j) / \beta)$ as in the case of reward-weighted regression (RWR) or obtain the subset of $K$ highest rewarding responses as in the case of ReST or Best-of-N. Finally, these methods train via supervised next-token prediction on these filtered or weighted responses. Given a weighting function, $F(\bx_i, \by_i^j |\by_i^{0\cdots N})$ that maps a response $\by_i^j$ for a given prompt $\bx_i$ to a scalar value conditioned on other responses $\by_i^k$ sampled from the model for the same prompt $\bx$, these methods maximize: 
\begin{align*}
    \max_\theta~ \mathbb{E}_{\bx \sim \mathcal{D}_\text{pref}, \by^{0 \cdots N} \sim \pi_{\theta^\text{old}}} \left[\log \pi_\theta(\by^i|\bx) \cdot F(\bx, \by^i |\by^{0\cdots N}) \right].\!\!\!
\end{align*}
These algorithms employ sample reuse because they operate in a \textbf{``batched'' online fashion}: instead of performing \textbf{\emph{exactly one}} gradient step on a given model sample; RWR, ReST, and SuperHF run more gradient updates, after which new samples are drawn. However, \textbf{since these methods only maximize likelihood (i.e., only positive multipliers), there is no negative gradient effect}. 

\textbf{Fully offline methods} like DPO and IPO run contrastive training on the preference dataset $\mathcal{D}_\text{pref}$ without any on-policy sampling. These methods train using variants of \cref{eq:reward_learning} (objective for IPO is shown \cref{subsection:ipo}) combined with \cref{eq:contrastive_parameterization} on responses $\by_w$ and $\by_l$ from the preference dataset $\mathcal{D}_\text{pref}$. Despite no on-policy sampling, contrastive loss between winning and losing responses explicitly attempts to reduce log-likelihood ratio $\log \left(\frac{\pi_\theta(\by|\bx)}{\piref(\by|\bx)}\right)$ for $\by_l$. Another offline method is Pref-FT~\citep{dubois2024alpacafarm} which runs supervised fine-tuning on preferred responses. These methods in general are akin to \textbf{offline model-free} methods, in that no reward model is utilized by these methods.

\vspace{-0.25cm}
\section{Research Questions and Analysis Setup}
\label{sec:questions}
\vspace{-0.25cm}

Our goal is to understand the behaviors of various procedures for fine-tuning language models. As discussed above, typically these methods differ along the use of on-policy sampling (with additional differences pertaining to sample reuse) and the presence of a negative gradient. We build a setup to understand these differences empirically by answering the following questions:

\textcolor{lightblue}{\textbf{Question 1:}} When does on-policy sampling improve over offline fine-tuning, even though on-policy samples are annotated by a reward model, which itself is learned from offline data? Is sample reuse useful or harmful for on-policy methods?

\textcolor{lightblue}{\textbf{Question 2:}} When does an explicit negative gradient help the discovery of effective policies compared to maximum likelihood approaches such as distilling the Best-of-N policy?

\textcolor{lightblue}{\textbf{Question 3:}} Does on-policy sampling offer complementary benefits to negative gradient, resulting in better performance with effective contrastive approaches (e.g., DPO)?

To gain practically useful and actionable insights, we must answer these questions in the context of coverage and geometric relations between the training data, reference policy, and the reward function. These relations affect the shape of the optimally fine-tuned policy and dictate the dynamics of various objectives under consideration. We consider specific conditions and relations that we discuss next.

\vspace{-0.25cm}
\subsection{Coverage Conditions and Geometric Relationships}
\label{sec:conditions}
\vspace{-0.2cm}

The dynamics of the KL-constrained surrogate optimization problem (\cref{eq:rlhf_objective}) depends on the geometric alignment between the ground-truth reward function $r^*$ and the reference policy initialization $\piref$ (see \cref{fig:teaser}). When the surrogate reward model $r_\phi$ is learned from the preference data, the coverage of the preference data used to train this reward model relative to the reference policy $\piref$ also dictates the correctness of reward estimates and hence controls the efficacy of the surrogate fine-tuning optimization. Likewise, the performance of purely offline methods (e.g., offline best-of-N or contrastive methods such as offline DPO) that do not use a reward model also depends on the relative geometric alignment between $r^*$ and $\piref$ (i.e., a smaller alignment would necessitate more deviation from $\piref$) and also on the relative coverage of preference data (i.e., the lower the coverage, the harder it is to discover high-reward responses). To understand the efficacy of various methods, we consider multiple scenarios that differ along these two factors: 
\begin{itemize}
\vspace{-0.2cm}
    \item \scenario{1}: the geometric alignment between the ground-truth reward function $r^*$ and the reference $\piref$, that can be measured in terms of any probabilistic divergence $\mathrm{D}(\piref, \exp(r^*))$. This concept is analogous to that of a \textbf{``concentrability coefficient''}~\citep{munos2008finite}.
    \item \scenario{2}: the coverage of the preference data used to train the surrogate reward model $r_\phi$ relative to the reference policy $\piref$, that can be measured in terms of the average density of the responses in the preference dataset under the reference policy initialization, $\piref$.
\vspace{-0.2cm}
\end{itemize}
Understanding the behavior of various approaches as a function of these factors will allow us to better understand the performance of various approaches on downstream fine-tuning in terms of problem geometry~\scenario{1} and statistical learning considerations~\scenario{2}.

\vspace{-0.3cm}
\subsection{Tasks and Datasets} 
\label{sec:tasks}
\vspace{-0.2cm}

We construct a variety of didactic and LLM tasks that allow us to gain intuition for performance of different methods under various scenarios grouped along relationships \scenario{1} and \scenario{2}.

\begin{wrapfigure}{r}{0.45\textwidth}
\centering
\vspace{-0.45cm}
\includegraphics[width=0.99\linewidth]{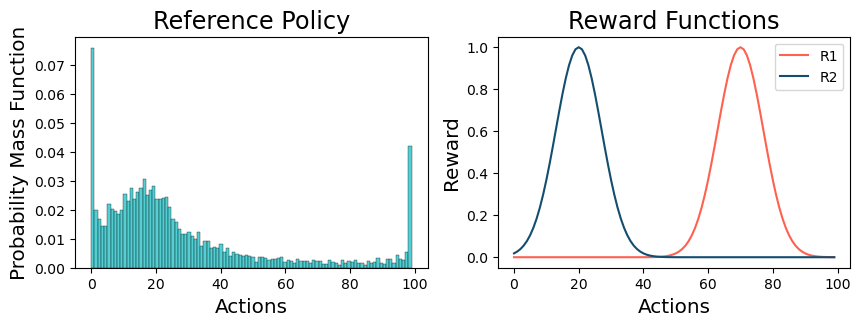}
\vspace{-0.3cm}
\caption{\label{fig:bandit_problem_setup} \footnotesize{\textbf{The didactic bandit problem} which we use for our analysis in this paper. Reference policy initialization and reward slice for each token (the total reward is a mean of token-level rewards). The optima of reward functions $\mathbf{R}_1$ and $\mathbf{R}_2$} occur in low-density and high-density regions respectively.}
\vspace{-0.35cm}
\end{wrapfigure}
\textbf{Didactic $N$-d bandit problems.} \cref{eq:rlhf_objective} poses preference fine-tuning as a KL-regularized contextual bandit problem over contexts $\bx$. Therefore, we develop a didactic $N$-dimensional contextual bandit problem. We use a set of tokens of size $V$ of size $100$. The context, $\bx$, is a single discrete token from $V$. A response $\ba$ is a sequence of $N = 10$ discrete tokens from $V$. We primarily study the effect of geometric relationship \scenario{1} and assume that the reward function is known exactly, therefore not accounting for the data coverage and training of the reward model. We consider two reward functions that differ in their relative geometry relative to the reference policy, as shown in \cref{fig:bandit_problem_setup}. Specifically, the difference lies in how perfectly the optimum of the reward function aligns with the high-density regions of the reference distribution. The optimum of the reward function $\mathbf{R}_1$ is located in low likelihood regions of the reference policy, whereas the optimum of $\mathbf{R}_2$ is roughly aligned with the mode of the reference policy. We hypothesize that on-policy sampling will be crucial to optimize reward function $\mathbf{R}_1$, whereas offline or maximum likelihood methods could be sufficient for the optimization of $\mathbf{R}_2$. 

\begin{wrapfigure}{r}{0.45\textwidth}
\centering
\vspace{-0.45cm}
\includegraphics[width=0.99\linewidth]{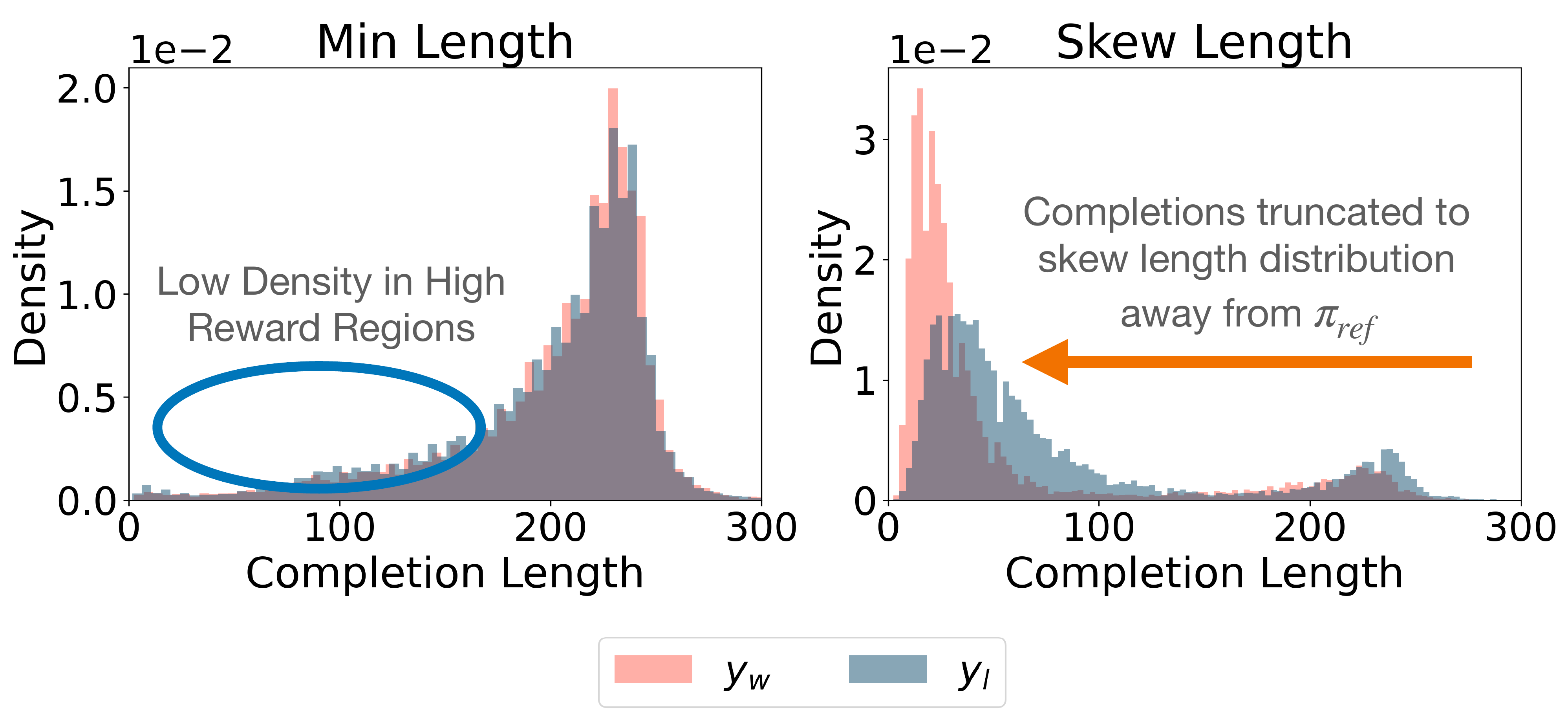}
\vspace{-0.7cm}
\caption{\footnotesize{\textbf{Word length distribution.} Above, we show the word length distribution for the preferred and dispreferred completions of the \textbf{Left:} min and \textbf{Right:} skew synthetic LLM datasets.}}
\label{fig:skew_dist}
\vspace{-0.5cm}
\end{wrapfigure}
\textbf{Synthetic LLM fine-tuning problems.}
Next, we will generalize our intuitions from bandit problems to the LLM setting. Instead of directly experimenting with human preferences, we first study two synthetic problems that utilize hand-crafted reward functions, which can be approximated via reward models. {Access to functional forms of these hand-crafted reward functions will enable us to track the ground-truth objective throughout training to see if our insights about various approaches under condition \scenario{1} will hold even when learning against a reward model. Subsequently, we run this experiment with an altered skewed preference data distribution (see Figure~\ref{fig:skew_dist}) to understand the effect of coverage conditions \scenario{2}. We consider two reward functions: \textbf{(1)} one that minimizes the response length (``\textbf{Min Length}''), analogous to $\mathbf{R}_1$ in the bandit problem, and \textbf{(2)} that attempts to anchor the response length to a pre-specified target value (``\textbf{Avg Length}''), which lies in the mode of the target distribution. This second condition exhibits similar characteristics to $\mathbf{R}_2$. The \textbf{Skew Length} scenario skews the preference data in the \textbf{Min Length} problem scenario.

\textbf{Full-scale LLM fine-tuning.} Finally, we scale up our study to full-scale LLMs, with real preference data. Recent work~\citep{singhal2023long} shows that preference labels are usually biased towards much longer responses, indicating that preference fine-tuning usually admits a geometric relationship where the mode of the reward function is distinct from the mode of human data (and hence, any reference policy). For the majority of our experiments, we use preference datasets from the AlpacaFarm benchmark~\citep{dubois2024alpacafarm}. We also scale up our experiments to UltraChat~\citep{ding2023enhancing}, a $\sim 10$ times larger dataset with responses from many strong LLMs such as GPT 4 and GPT-3.5.

\vspace{-0.2cm}
\subsection{A Generic Fine-Tuning Algorithm Encapsulating All Axes}
\vspace{-0.2cm}

To systematically analyze the behavior of fine-tuning methods that differ along the axes discussed in \cref{subsec:characterization}, in this section, we introduce a generic algorithm with different hyperparameters associated with each axes. With a generic algorithm of this sort, we will be able to answer our research questions by varying each hyperparameter. 
Our unified practical algorithm is shown \cref{alg:onpolicy}. While on-policy algorithms perform steps 1 and 2 of on-policy data collection with a reward model, purely offline methods (e.g., DPO and RWR) utilize preference data directly.

\begin{algorithm}
\caption{A Unified Fine-Tuning Algorithm}
\label{alg:onpolicy}
\begin{algorithmic}
\For{training iterations}
    \State  (1) Sample $B / C$ prompts $[\bx_1, \bx_2, \cdots, \bx_{B/C}]$.
    \State  (2) Generate dataset $\mathcal{D}$ with $C$ responses for $\frac{B}{C}$ prompts, from the policy\\ ~~~~~~~~~~~for online ($\by_i^1, \by_i^2, \cdots, \by_i^C \sim \pi_\theta(\cdot|\bx_i)$) or from an offline dataset ($\by_i^1, \by_i^2, \cdots, \by_i^C \sim \mathcal{D}_{\text{pref}}$).
    \State (3) If applicable, label the responses $\by_i^1, \by_i^2, \cdots, \by_i^C$, with rewards drawn \\ ~~~~~~~~~~~from the learned reward model $\widehat{r}_\phi(\by | \bx)$
    \For{$T$ inner iteration steps}
        \State (a) Divide $\mathcal{D}$ into mini-batches $\mathcal{D}_1, \ldots, \mathcal{D}_{N}$, each with $M$ prompts-response pairs
        \For{$i = 1, \ldots, N$}
            \State (i) Apply the gradient of the objective $\mathcal{L}(\theta; \mathcal{D}_i; \widehat{r}_\phi)$ prescribed by the fine-tuning method.
        \EndFor{}
    \EndFor{}
\EndFor{}
\end{algorithmic}
\end{algorithm}
\vspace{-0.2cm}

To study the impact of on-policy sampling, we vary the extent to which updates are made on data from the current policy. We can control this by two means in \cref{alg:onpolicy}: \textbf{(1)} by varying the total number of samples $|\mathcal{D}| = \frac{B}{C} \times C = B$ used for a given training iteration assuming the algorithm performs exactly one pass over all this sampled data while keeping the \textbf{minibatch size $M$ fixed}, and \textbf{(2)} by varying the number $T$ of gradient steps performed on a given set $\mathcal{D}$ of on-policy samples (i.e., a larger $T$ leads to more off-policy updates). In other words, approach \textbf{(1)} will perform more updates using stale data for large values of $|\mathcal{D}|$; and for small values of $|\mathcal{D}|$, approach \textbf{(2)} will make more off-policy updates if $T$ is larger. While both approaches enable us to control how on-policy an algorithm is, approach \textbf{(1)} does not reuse samples (since $\mathcal{D}$ is large), but approach \textbf{(2)} reuses samples for different number of gradient updates, controlled directly by $T$. By studying both approaches for inducing off-policyness, we can isolate the effect of sample reuse on on-policy methods. We also study offline methods with no on-policy sampling, such as DPO, and filtered supervised learning on the preferred response $\by_w$ in the dataset to understand the role of the negative gradient.

\vspace{-0.25cm}
\section{Empirical Analysis Results}
\vspace{-0.2cm}

In this section, we will present the results of our empirical study to answer our research questions. To answer each question, we will begin by studying the didactic bandit problem with the ground-truth reward function, followed by synthetic and then full-scale LLM fine-tuning problems.

\vspace{-0.2cm}
\subsection{Question 1: The Role of On-Policy Sampling}
\label{sec:question1}
\vspace{-0.15cm}

To understand the role of on-policy sampling, we will investigate if on-policy sampling can improve performance for several approaches followed by making conclusions regarding sample reuse. 

\vspace{-0.25cm}
\subsubsection{Takeaway 1: On-Policy Sampling in the Reward Model Improves Performance}
\vspace{-0.2cm}

We first study on-policy sampling as a function of the geometric relationship \scenario{1} in our bandit setting (see \cref{fig:bandit_problem_setup}), with no sampling error. Then, we will extend our conclusions to the LLM setting.

\begin{figure}[h!]
\vspace{-0.2cm}
    \centering
    \includegraphics[width=\columnwidth]{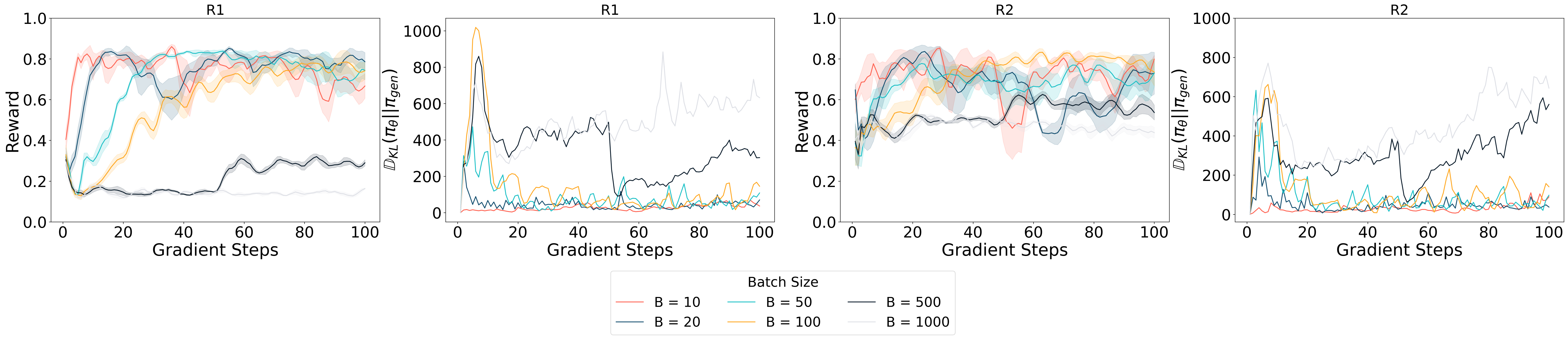}
    \vspace{-0.7cm}    \caption{\label{fig:transformer_batch_size_on_policy}\footnotesize{\textbf{On-policy sampling on bandit problems.} Performance of on-policy best-of-N as a function of the data sampled in each iteration. Larger batch sizes result in more off-policy updates. \textbf{Left}: (i) reward vs update step for $\mathbf{R}_1$, (ii) divergence between the policy parameters and data collection policy during training; \textbf{Right:} (i) reward vs update step for $\mathbf{R}_2$, (ii) KL divergence for $\mathbf{R}_2$. Observe the slow learning speed of more off-policy updates in $\mathbf{R}_1$, but less severe degradation for $\mathbf{R}_2$, where peaks in the reference policy and reward function are more aligned.}}
    \vspace{-0.25cm}
\end{figure}
\textbf{Didactic bandit problems.} ~\cref{fig:transformer_batch_size_on_policy} shows that given a fixed amount of total data budget, \emph{sampling data more frequently from more recent policies}, but in smaller batches, results in better performance with both $\mathbf{R}_1$ and $\mathbf{R}_2$. Doing so, naturally makes the algorithm more on-policy since each gradient update uses a mini-batch sampled from a more recent policy. This is also reflected in larger values of divergences between the sampling policy $\pi_\text{gen}$ and the policy $\pi_\theta$, $\mathbb{D}_{\text{KL}}(\pi_\theta || \pi_\text{gen})$, in~\cref{fig:transformer_batch_size_on_policy}. Concretely, larger $B$ results in higher peak values of this divergence during training indicating further deviation from the data at intermediate times during training. This means that being more on-policy corresponds to better performance and faster convergence for best-of-N. %

That said, we also note in \cref{fig:transformer_batch_size_on_policy} that the performance degradation with more off-policy updates is substantially milder for $\mathbf{R}_2$, indicating that when the peak in the reward function lies in the high likely regions of the reference policy, a higher degree of off-policy updates is tolerable.

\begin{table}[ht!]
\centering
\begin{tabular}{c || c  c} 
 \toprule
\scenario{1} $\downarrow$ ~~~ $\vert\vert$ ~~~ \scenario{2} $\rightarrow$  & high $\mathcal{D}_\text{pref}$ and $\pi_\text{ref}$ overlap & low $\mathcal{D}_\text{pref}$ and $\pi_\text{ref}$ overlap\\
\midrule
peaks of $r^*$ and $\pi_\text{ref}$ overlap & $\checkmark$ \textbf{Mode Length} & $\times$ \\
peaks of $r^*$ and $\pi_\text{ref}$ disjoint & $\checkmark$ \textbf{Min Length}& $\checkmark$ \textbf{Skew Length} \\
\bottomrule
\end{tabular}
\vspace{-0.25cm}
\caption{\footnotesize{\textbf{Coverage conditions and geometric relations} that we study with synthetic LLM fine-tuning data. The three settings we study differ in terms of overlap between $\pi_\text{ref}$, reward $r^*$, and the preference dataset, $\mathcal{D}_\text{pref}$.}}
\label{table:conditions}
\vspace{-0.25cm}
\end{table}
\textbf{Synthetic LLM problems.} In this problem setting, we optimize the policy against a reward model, which is learned from preference data. Per \cref{sec:tasks}, we construct three scenarios that differ along geometric (\scenario{1}) and coverage (\scenario{2}) conditions as depicted in \cref{table:conditions}. The peak of the reward in the \textbf{Min Length} scenario appears in the less likely regions of  $\pi_\text{ref}$, whereas the peak of the reward function in the \textbf{Mode Length} scenario appears in highly likely regions under $\pi_\text{ref}$. 

\begin{figure}[h!]
\vspace{0.05cm}
    \centering
    \includegraphics[width=0.8\columnwidth]{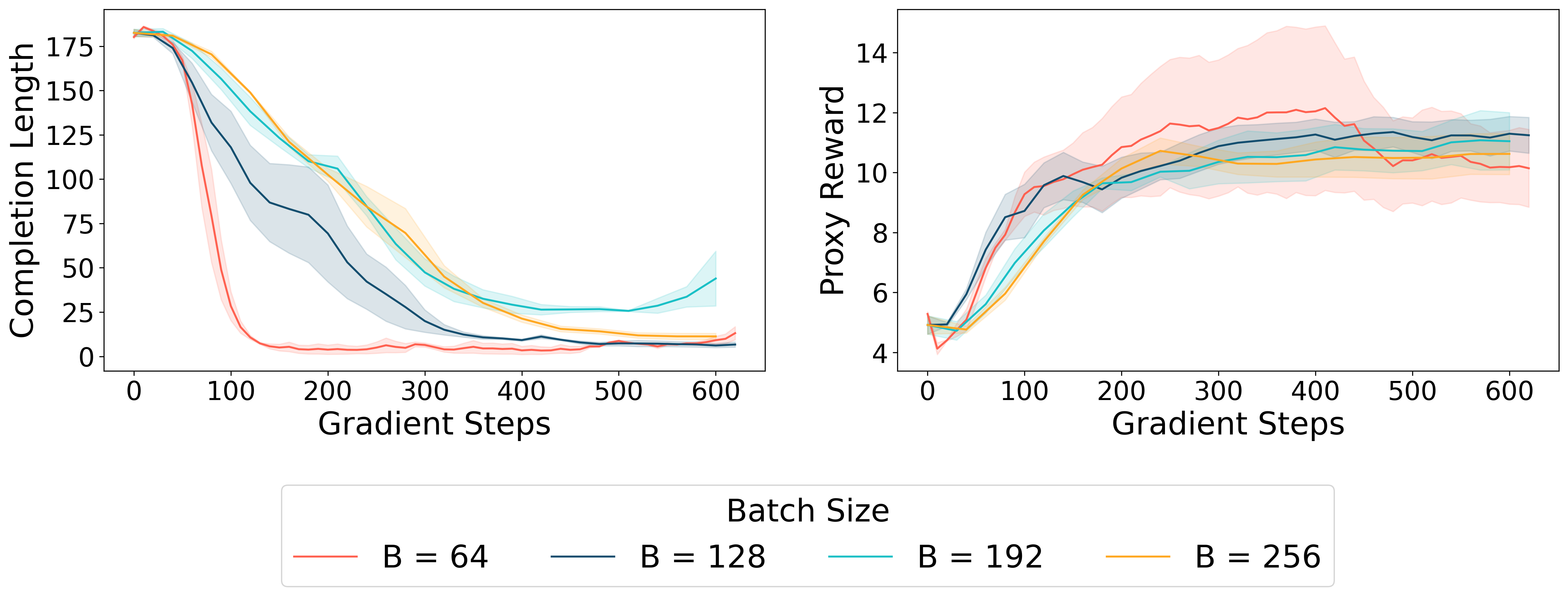}
    \vspace{-0.4cm}
\caption{\label{fig:min_length_batch_size_on_policy}\footnotesize{\textbf{On-policy sampling for PPO in the Min Length scenario.}  This plot keeps the minibatch size $M$ fixed to 64, but samples more stale data when $B$ is large. Increasing $B$ results in more off-policy updates and consequently slower convergence to ground-truth reward (i.e., a completion length of $0$). \textbf{Left}: average completion length (lower the better), and \textbf{Right}: proxy reward vs gradient steps. Being more on-policy results in better performance. The mini-batch size $M$ used for gradient updates is kept fixed to avoid confounders arising from the use of stochastic optimization procedures.}}
    \vspace{-0.2cm}
\end{figure}

We present our results for one algorithm in detail (in this case, PPO) (\cref{fig:min_length_batch_size_on_policy,fig:mode_length_batch_size_on_policy,fig:skew_length_batch_size_on_policy}) and then present a summary plot showing that our conclusions also transfer to other algorithms (such as REINFORCE and RWR) (\cref{fig:synthetic_llm_batch_size_on_policy}). Extending insights from the bandit problem, in the \textbf{Min Length} scenario, we find that \textbf{being more on-policy (i.e., a smaller $B$) leads to a lower completion length and hence a higher gold reward, despite potential inaccuracies in the proxy reward model} that PPO is actually optimizing (\cref{fig:min_length_batch_size_on_policy}). Akin to our bandit experiments, we also observe that smaller batch sizes ($B = 64$ and $B = 128$) optimize the proxy reward at a faster rate compared to $B = 192$ and $B = 256$. This indicates that with a significant overlap between the preference data and the reference policy, on-policy sampling still leads to better performance with fewer updates. 
We also find similar trends across on-policy variants of RWR and REINFORCE, where modulo training instabilities, being more on-policy results in better performance (\cref{fig:synthetic_llm_batch_size_on_policy}; \textbf{Min Length}).

\begin{figure}[h!]
\vspace{-0.15cm}
    \centering
    \includegraphics[width=0.8\columnwidth]{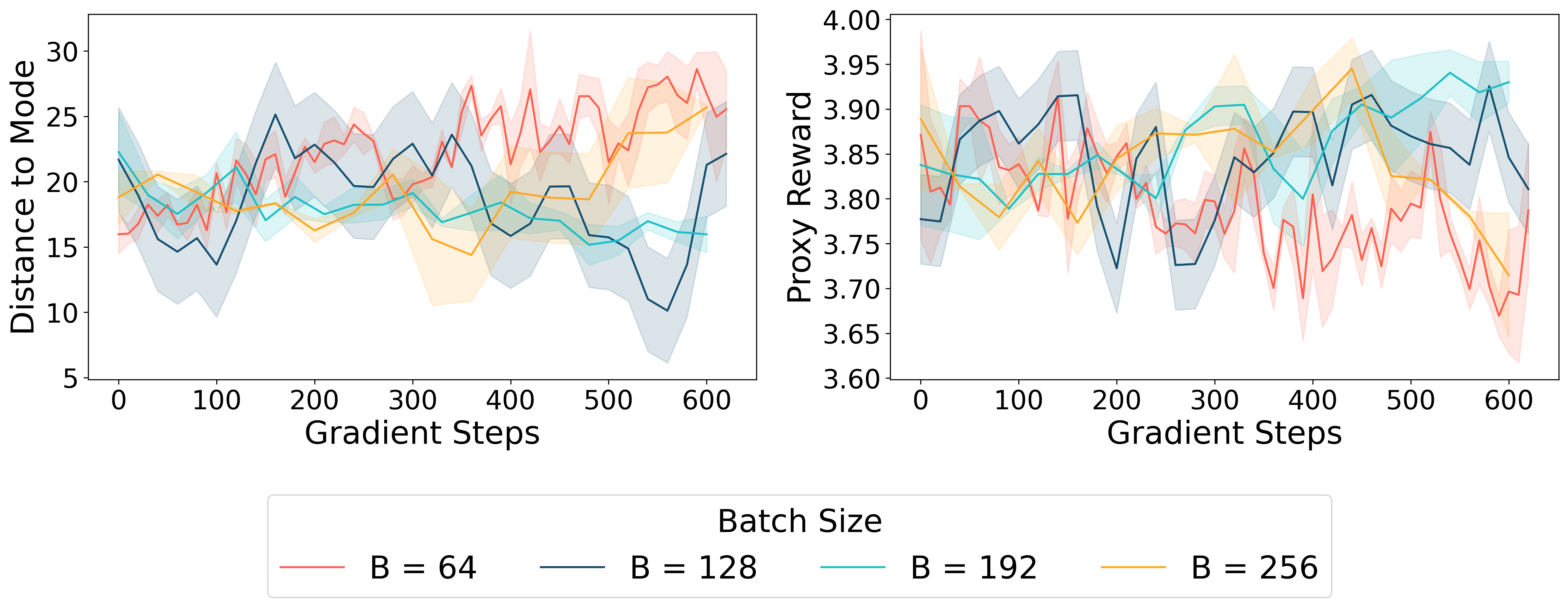}\vspace{-0.4cm}\caption{\label{fig:mode_length_batch_size_on_policy}\footnotesize{\textbf{On-policy sampling for PPO in the Mode Length scenario.} In this case, since the peak in the reward function and the highly likely regions of the reference policy are close, we find that the degree of on-policyness does not significantly affect performance. \textbf{Left}: distance to mode i.e., |completion length - average length in the dataset| (lower the better), \textbf{Right}: proxy reward vs gradient steps. As optimal policy $\pi^*$ and reference policy $\piref$ are very close to each other in this scenario, we don't see any significant performance gains from being on-policy. The mini-batch size $M$ used for the gradient update is kept fixed.}}
    \vspace{-0.35cm}
\end{figure}

In the \textbf{Mode Length} scenario, where the preferred response for each preference pair are those that are closest to the average length in the dataset (203), varying the degree of on-policy sampling by adjusting the sampling frequency largely does not affect either the proxy or gold reward for PPO (\cref{fig:mode_length_batch_size_on_policy}). We make similar observations for other algorithms: \cref{fig:synthetic_llm_batch_size_on_policy}; \textbf{Mode Length}: different degrees of on-policyness perform similarly, except the more on-policy runs sometimes exhibit instability. This is in agreement with the results from the bandit setting above: \textbf{when the peak in the reward function lies in highly likely regions under the reference policy,} \textbf{on-policy sampling has minor effect and more off-policy configurations of the algorithm can perform similarly too.}

\begin{figure}[h!]
\vspace{-0.2cm}
    \centering
    \includegraphics[width=0.8\columnwidth]{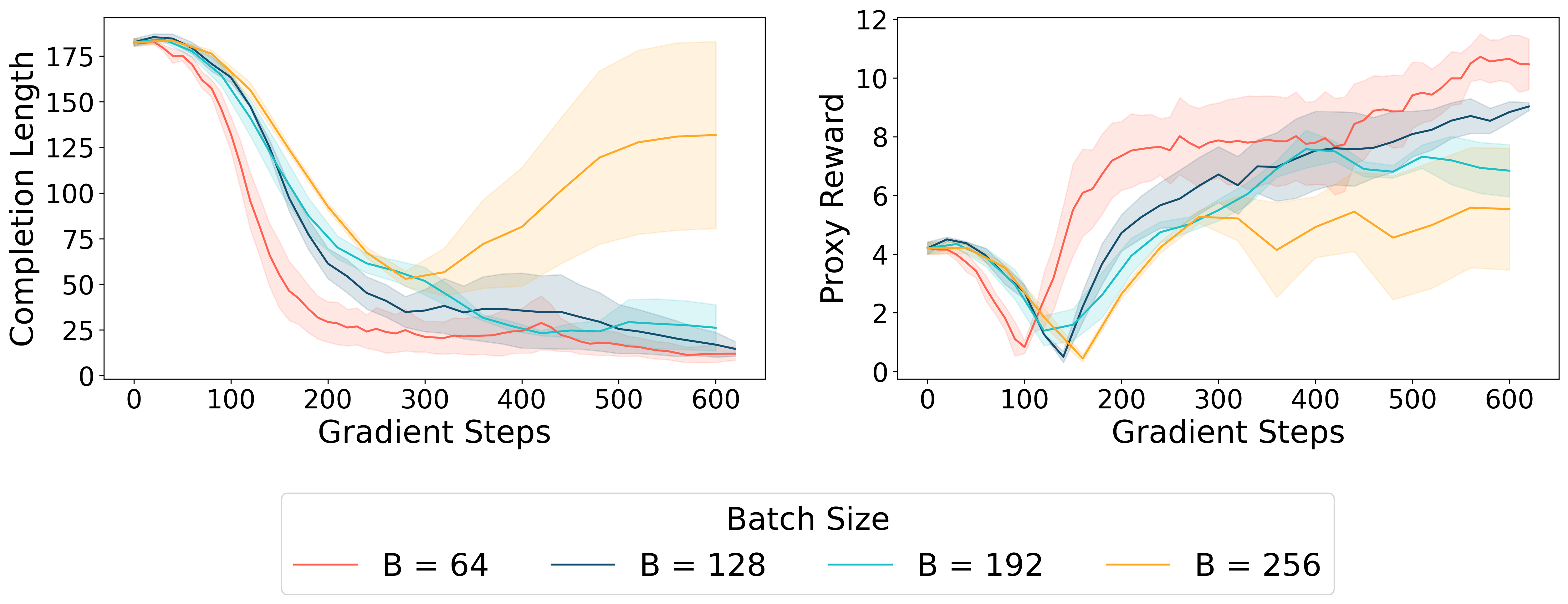}    \vspace{-0.4cm}\caption{\label{fig:skew_length_batch_size_on_policy}\footnotesize{\textbf{On-policy sampling for PPO on the Skew Length scenario.} Being more on-policy results in faster convergence and better performance. \textbf{Left}: average completion length (lower the better), and \textbf{Right}: proxy reward vs gradient steps. Being more on-policy results in better performance.}}
    \vspace{-0.2cm}
\end{figure}

Finally, to evaluate the robustness of these findings under more challenging coverage conditions, we deliberately skew the length distribution in the preference dataset to make it distinct from the reference policy (called \textbf{Skew Length}). Concretely, with a 95\% probability, we truncate the length of the response by sampling a length from an exponential distribution, which naturally leads to a shorter completion length. The remaining 5\% of samples are drawn from the standard SFT policy to simulate the broader coverage for the preference data. Overall, the resulting data admits a significantly skewed distribution over response lengths, as visualized in \cref{fig:skew_dist}. Not only does the peak in the reward function now appear in less likely regions of the reference policy, but to succeed, an optimization algorithm must now do the required heavy lifting to shift the probability mass to the low-density regions of the response space that maximize reward.

\begin{figure}[h!]
\vspace{-0.15cm}
    \centering
    \includegraphics[width=0.99\columnwidth]{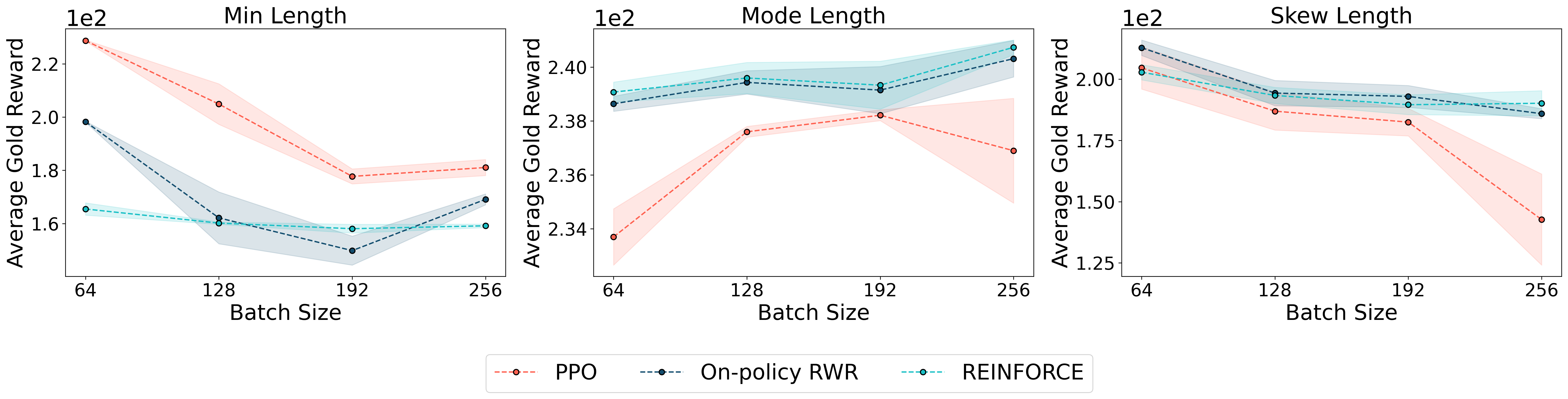}
    \vspace{-0.4cm}\caption{\label{fig:synthetic_llm_batch_size_on_policy}\footnotesize{\textbf{Summary: effect of on-policy sampling on synthetic LLM problems.} Average gold reward over the course of training for RWR, and REINFORCE with different $B$. For \textbf{Min Length} and \textbf{Skew Length}, generally being more on-policy (i.e., smaller batch size) leads to a higher gold reward. For \textbf{Mode Length}, all batch sizes perform close to each other (note that the range of the $y$-axis is small), with performance differences largely due to instability.}}
    \vspace{-0.15cm}
\end{figure}

Our detailed results of running PPO in this setting are shown in \cref{fig:skew_length_batch_size_on_policy}. In this setting, we still find that more on-policy updates lead to a higher gold reward with PPO. In addition, we also observe much larger gaps in proxy reward values attained at any given gradient step compared to the \textbf{Min Length} scenario, in favor of on-policy sampling. For other algorithms, we also observe strong and clear trends supporting that on-policy sampling with a smaller but frequently sampled batch results in better performance as shown in the summary plot (see \cref{fig:synthetic_llm_batch_size_on_policy}; \textbf{Skew Length}).

\begin{figure}[h!]
\vspace{-0.1cm}
    \centering
    \includegraphics[width=0.75\columnwidth]{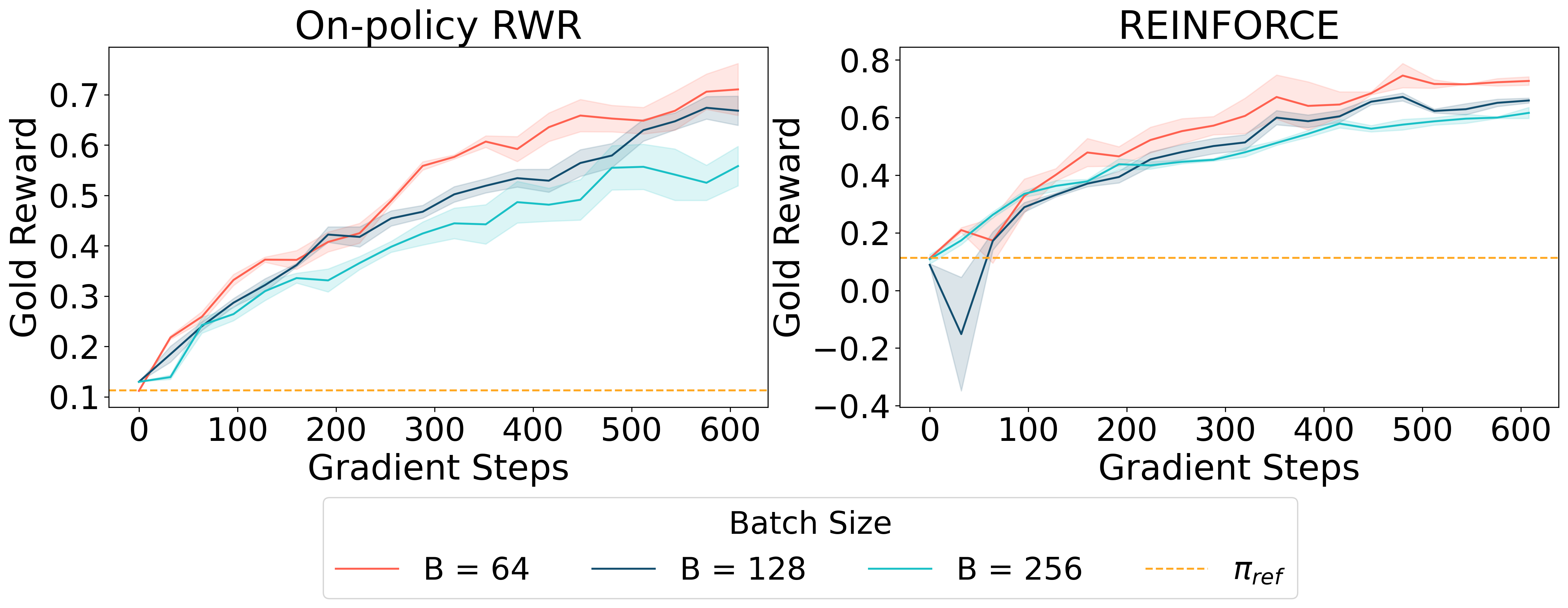}
    \vspace{-0.35cm}    \caption{\label{fig:alpacafarm_batch_size_ablation_rwr_plus_reinforce}\footnotesize{\textbf{Effect of on-policy sampling on AlpacaFarm} with a fixed mini-batch, but varying batch size $B$, for (\textbf{Left}) on-policy RWR and (\textbf{Right}) REINFORCE. Increasing $B$ makes updates more off-policy and this results in lower performance.}}
    \vspace{-0.2cm}
\end{figure}

\textbf{Full-scale LLM problems.} Finally, we evaluate if our insights transfer to the full-scale AlpacaFarm setup. We use a Pythia-1.4B model as our reference policy and generate two responses per prompt. We label the preferred and dispreferred responses with a gold reward model of human preferences from AlpacaFarm to construct a preference dataset. \cref{fig:alpacafarm_batch_size_ablation_rwr_plus_reinforce} shows that our intuitions from the simple bandit and synthetic LLM experiments transfer to this real preference learning task, as making updates on only on-policy samples leads to higher gold reward for both on-policy RWR and REINFORCE.

\vspace{-0.25cm}
\subsubsection{Takeaway 2: On-Policy Sample Reuse Can Enable Leveraging Off-Policy Data}
\vspace{-0.15cm}

In the previous section, exactly one gradient step was taken on a given sample and we found that making updates on stale data was not helpful due to off-policy updates. \textbf{Is there any scenario under which we can still attain good policy performance despite employing off-policy updates?} In this section, we will answer this question, and show that it might be possible to learn with off-policy updates for some algorithms if we are allowed to make more than one update on a given sample. Of course, a substantial amount of sample reuse is detrimental since it would lead to more off-policy updates, thus leading to statistical or even propensity overfiting~\citep{swaminathan2015self} for some methods, but it is reasonable to surmise that some amount of sample reuse can help. To study sample reuse, we compare methods when $T>1$ gradient steps can be made on a given sample. 

\begin{figure}[h]
\vspace{-0.2cm}
    \centering
    \includegraphics[width=0.99\columnwidth]{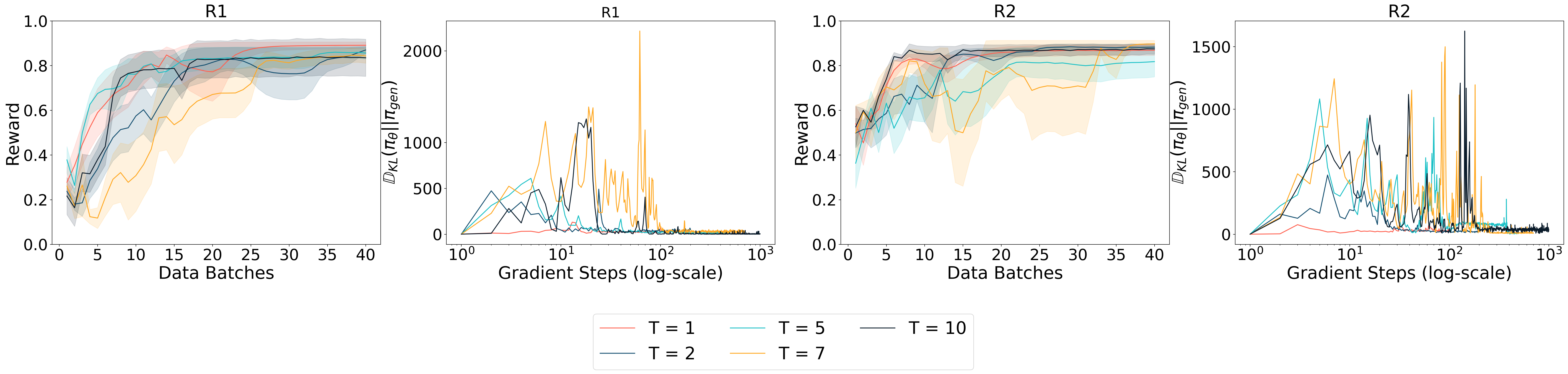}
    \vspace{-0.45cm}
\caption{\label{fig:transformer_on_policy_sampling_performance_vs_gradient_steps}\footnotesize{\textbf{Effect of on-policy sample reuse on bandit problems.} Reward vs gradient steps for a different number of inner iteration steps, $T$, on the same data batch for RWR. Increasing $T$ controls the number of gradient steps taken before collecting the new batch of on-policy samples. We observe non-monotonic performance trends while varying $T$.}}
    \vspace{-0.35cm}
\end{figure}

We study sample reuse for on-policy RWR in the bandit setting in \cref{fig:transformer_on_policy_sampling_performance_vs_gradient_steps}. While increasing $T$ can slow down convergence in general, we note that using a larger value of $T$ may be better (e.g., $T = 5$ learns faster than $T=2$; $T=10$ learns faster than $T=7$).

\textbf{Synthetic LLM problems.} We also evaluate the effect of sample reuse on synthetic LLM problems. In this case, we study two algorithms PPO and on-policy best-of-N to be able to understand the effect of sample reuse on multiple algorithms. In contrast to the performance degradation with off-policy updates induced due to stale samples in PPO, we find that off-policy updates induced due to sample reuse do not hurt performance (\cref{fig:llm_length_t_ablation}; PPO), with even $T=8$ performing similarly to $T=1$. On the other hand increasing $T$ from $1$ to $2$, i.e., performing two gradient updates on each sample improves the golden reward for best-of-N (\cref{fig:llm_length_t_ablation}; Best-of-N) within a given data sampling budget. 

\begin{figure}[h!]
\vspace{-0.2cm}
    \centering
    \includegraphics[width=0.99\columnwidth]{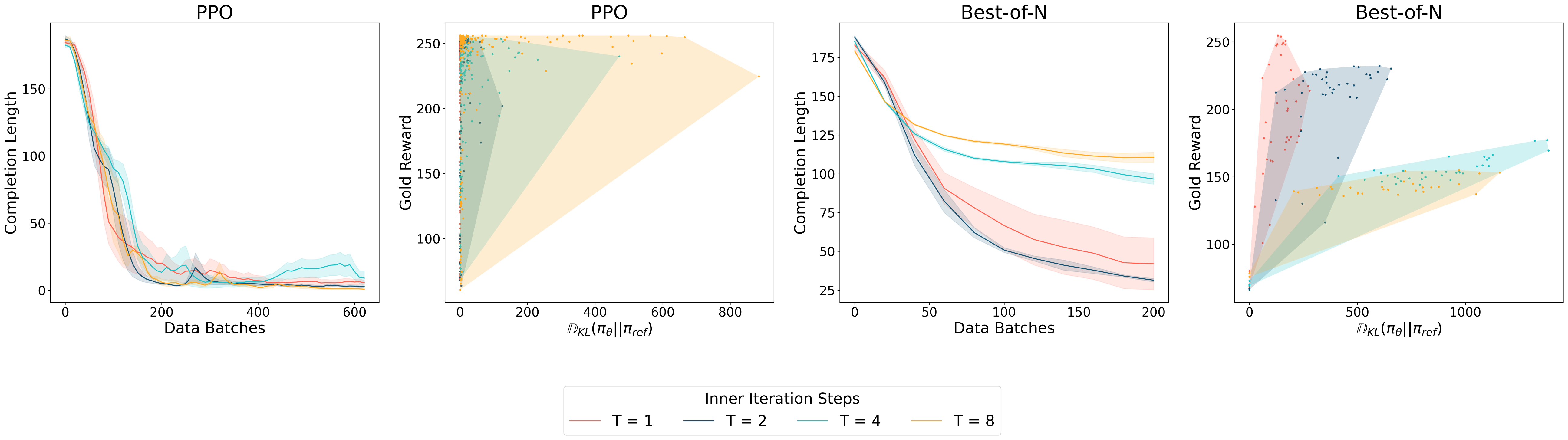}
    \vspace{-0.4cm}
\caption{\label{fig:llm_length_t_ablation}\footnotesize{\textbf{Effect of on-policy sample reuse in the Min Length scenario.} Average completion length (i.e., the lower the better) vs gradient steps for a different number of inner iteration steps, $T$, on the same data batch. A larger value of $T$ implies that the algorithm is more off-policy. Observe that some sample reuse can improve sample efficiency (T = 2 outperforms T = 1), but excessive sample reuse can hurt performance. Also note that algorithms with mechanisms to control off-policy updates such as PPO are suited to perform better in the off-policy sample reuse setting.}}
    \vspace{-0.2cm}
\end{figure}

\textbf{Why do PPO and best-of-N respond differently to sample reuse? }We believe that this is because PPO employs an off-policy correction, and hence, significantly off-policy samples do not contribute to the gradient, addressing the well-known challenge of propensity overfitting~\citep{swaminathan2015self}. This is not the case with on-policy best-of-N, where excessive sample reuse can hurt exploration, because training on old samples with a log-likelihood loss push the current policy to be close to the stale data-generating policy. That said, more than one gradient step can still be useful when presented with a fixed data budget, unless it bottlenecks exploration of high reward regions.

\begin{AIbox}{Takeaways for on-policy sampling}
On-policy sampling generally improves performance and efficiency, especially in cases when the peak of reward appears farther from the reference policy, even when the reward model is learned from the same preference dataset that methods without on-policy learning also use. In some cases, sample reuse can reduce the dependency on on-policy sampling of data, but it presents a tradeoff by reducing the exploration of the response space.
\end{AIbox}

\vspace{-0.2cm}
\subsection{Question 2: The Role of Negative Gradient}
\label{sec:question2}
\vspace{-0.15cm}

To understand the role of negative gradient, we will compare contrastive algorithms such as DPO and IPO with maximum likelihood methods such as RWR (or Pref-FT, which attempts to increase the likelihood of the preferred response only) and best-of-N in a fully offline setting, where no new on-policy samples are used. We will also aim to understand the mechanisms behind these methods.

\vspace{-0.2cm}
\subsubsection{Takeaway 1: Negative Gradient Enables Faster Convergence Amongst Offline Methods} \label{sec:negative_gradient}
\vspace{-0.1cm}

\begin{figure}[h]
    \centering
    \vspace{-0.2cm}
    \includegraphics[width=0.99\columnwidth]{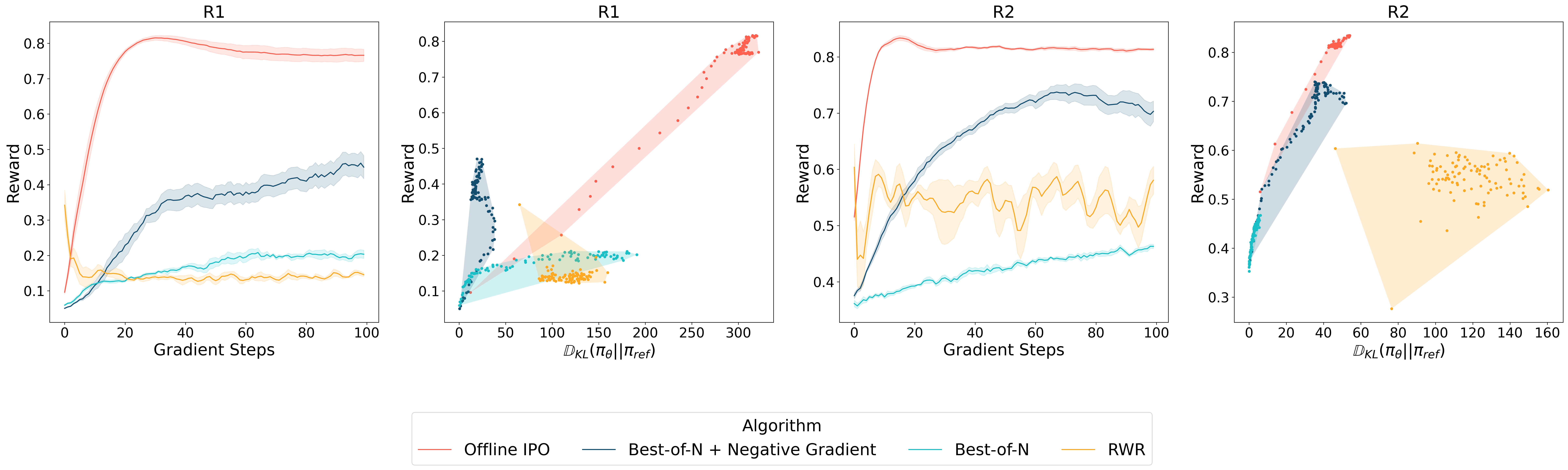}
    \vspace{-0.3cm}
    \caption{\label{fig:transformer_negative_gradients}\footnotesize{\textbf{Negative gradients on the didactic bandit problems.} Average reward during training and the KL-reward trade-off for four algorithms in the fully offline setting: best-of-N (no negative gradient), RWR (no negative gradient), best-of-N + an explicit negative gradient on dispreferred actions, and IPO (with negative gradient). Negative gradient helps find a better policy by aggressively pushing down the likelihood of bad actions, and this leads to larger KL values.}}
    \vspace{-0.1cm}
\end{figure}

We begin by comparing a representative set of offline algorithms on the didactic bandit problem. These methods include those that do not use a contrastive update on the didactic bandit problem, namely offline supervised approaches, Best-of-N and offline RWR, and offline IPO~\citep{2023arXiv231012036G}, a representative offline fine-tuning method which uses a contrastive negative gradient term. We also consider a variant of best-of-N where we explicitly add a term to the loss function that attempts to minimize the likelihood of the dispreferred response akin to unlikelihood~\citep{Welleck2020Neural} (see~\cref{section:banditAlgorithmsAppendix} for more details). In \cref{fig:transformer_negative_gradients}, we find that IPO and best-of-N + negative gradient learn a better policy from an offline dataset collected from sub-optimal $\piref$, compared to best-of-N and RWR. IPO achieves a better KL-reward trade-off in $\mathbf{R}_1$ (where high likelihood regions of $\piref$ and the peak in $r^*$ are far away from each other). While best-of-N attains a higher reward when the reward function is given by $\mathbf{R}_2$ (where the peaks in $\piref$ and $r^*$ overlap) compared to $\mathbf{R}_1$, it still underperforms IPO. We suspect that this is because maximizing likelihood on some responses alone is not enough to steer the learned policy away meaningfully away from $\piref$ towards the peak in the reward function, especially when this peak is far away from $\pi_\text{ref}$. Best-of-N + negative gradient significantly outperforms Best-of-N in both scenarios and closes the performance gap to IPO, which shows that explicitly adding a loss term to minimize the probability on dispreferred responses can provide a substantial performance improvement. That said, for reward function $\mathbf{R}_2$, we also observe a smaller gap between the best algorithm without a negative gradient (i.e., RWR) and offline IPO, indicating that when the peak in $\piref$ and $r^*$ exhibit more overlap, the performance benefits of contrastive training are smaller. We also investigated a simpler 1-token bandit problem where we found best-of-N to be better than IPO for $\mathbf{R}_2$. This is possibly due to the much smaller space of possible tokens and responses, where maximum likelihood methods perform well enough.

\begin{figure}[h!]
\vspace{-0.2cm}
    \centering
    \includegraphics[width=0.85\columnwidth]{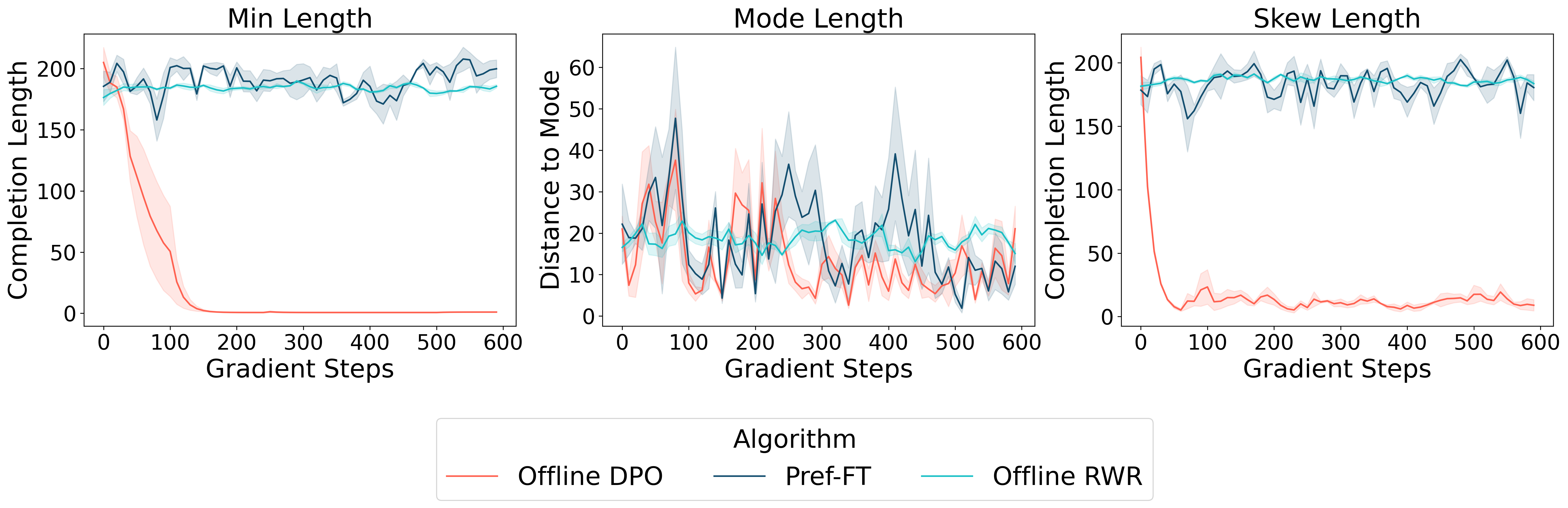}
    \vspace{-0.25cm}    
    \caption{\footnotesize{
    \textbf{Negative gradients in synthetic LLM problems.} Completion length (inverse of the true reward) for three offline algorithms. DPO outperforms Pref-FT and offline RWR in \textbf{Min Length} and the \textbf{Skew Length} settings, where the peak in $r^*$ and $\piref$ are misaligned. For the \textbf{Mode Length} setting, all of the algorithms perform similarly.}}
    \vspace{-0.25cm}
    \label{fig:neg_grad_length}
\end{figure}

\textbf{Synthetic LLM problems.} Our experiments in the synthetic LLM setting corroborate this finding. Here  we compare Pref-FT with DPO (with negative gradients). In the \textbf{Min Length} setting, we find in \cref{fig:neg_grad_length} that DPO significantly outperforms Pref-FT. On the other hand, when the peak in the ground-truth reward appears in high-likely regions of the reference policy and the preference data $\mathcal{D}_\text{pref}$ covers this region (\textbf{Mode Length}), we find both approaches to perform similarly. Finally, in the \textbf{Skew Length} scenario when $\piref$ and $\mathcal{D}_\text{pref}$ do not overlap significantly, but the peak in $r^*$ is covered by the preference dataset $\mathcal{D}_\text{pref}$, we also find that DPO is much more effective in driving the policy further from the reference initialization and outperforms Pref-FT.

\begin{wrapfigure}{r}{0.6\textwidth}
    \vspace{-0.4cm}
    \centering
    \includegraphics[width=0.99\linewidth,height=0.2\textwidth]{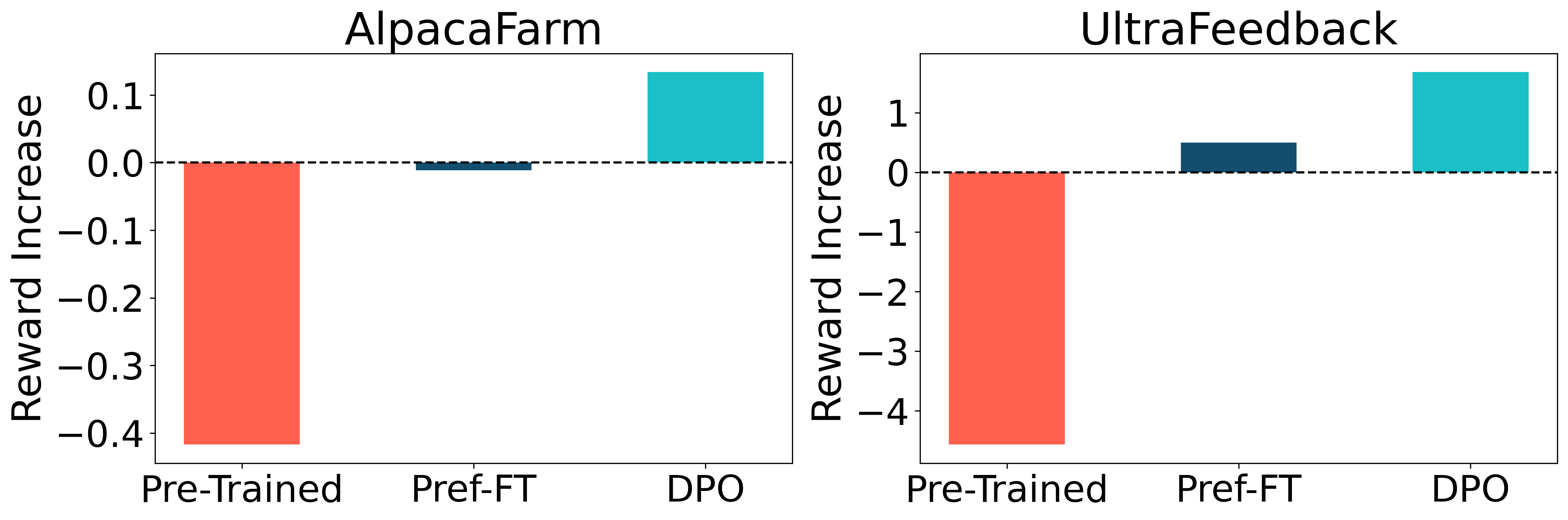}
    \vspace{-0.2cm}    
    \caption{\footnotesize{\textbf{Negative gradients in AlpacaFarm (left) and UltraFeedback (right) for offline methods.} We plot the increase in average gold reward \textbf{compared to the reference model} for different offline approaches. Algorithms with a negative gradient such as DPO outperform approaches such as Pref-FT not utilizing any negative gradient term. }}
    \vspace{-0.4cm}
    \label{fig:negative_gradients_full_scale}
\end{wrapfigure}
\textbf{Full-scale LLM fine-tuning.} Finally, we compare supervised Pref-FT and contrastive DPO when fine-tuning on actual preference data. In addition to AlpacaFarm, we also run experiments using the Ultra-Feedback~\citep{ding2023enhancing} dataset. For the Ultra-Feedback dataset, we use different models (GPT-3.5, GPT-4) to generate responses to various prompts. The resulting dataset has a broader preference
dataset distribution than $\piref$. We utilize a checkpoint of the Mistral7B model obtained by running supervised next-token prediction on a subset of UltraChat (comprising of GPT-3.5 responses) as the reference initialization. We use the UltraRM model with a LLaMA2-13B base architecture as our gold reward model.  
As shown in \cref{fig:negative_gradients_full_scale}, DPO which utilizes a negative gradient shows a much larger improvement over the reference policy $\piref$compared to methods that do not utilize a negative gradient (e.g., Pref-FT). 

\vspace{-0.2cm}
\subsubsection{Takeaway 2: Mechanisms Explaining the Behavior of the Negative Gradient}
\label{sec:mechanisms}
\vspace{-0.1cm}

Having seen that using a negative gradient leads to much better performance, we next attempt to understand the mechanism behind this better performance. To do so, we visualize the evolution of the log-likelihoods of the preferred response and the dispreferred response in a held-out dataset as multiple gradient steps are taken on an offline preference optimization loss. 

\textbf{Contrastive training increases the gap between the likelihoods of preferred and dispreferred responses.} Perhaps as expected, we find that DPO-style contrastive training is more effective at increasing the gap between the likelihoods of preferred and dispreferred responses compared to offline Pref-FT in several LLM settings: the synthetic LLM settings with \textbf{Min Length} and \textbf{Skew Length}, and full-scale AlpacaFarm and UltraFeedback settings (\cref{fig:negative_gradient_log_prob_difference}). More concretely, note that the margin for Pref-FT largely converges to 0, whereas offline DPO can enable a larger margin.

\begin{figure}[h!]
\vspace{-0.2cm}
    \centering
    \includegraphics[width=0.99\columnwidth]{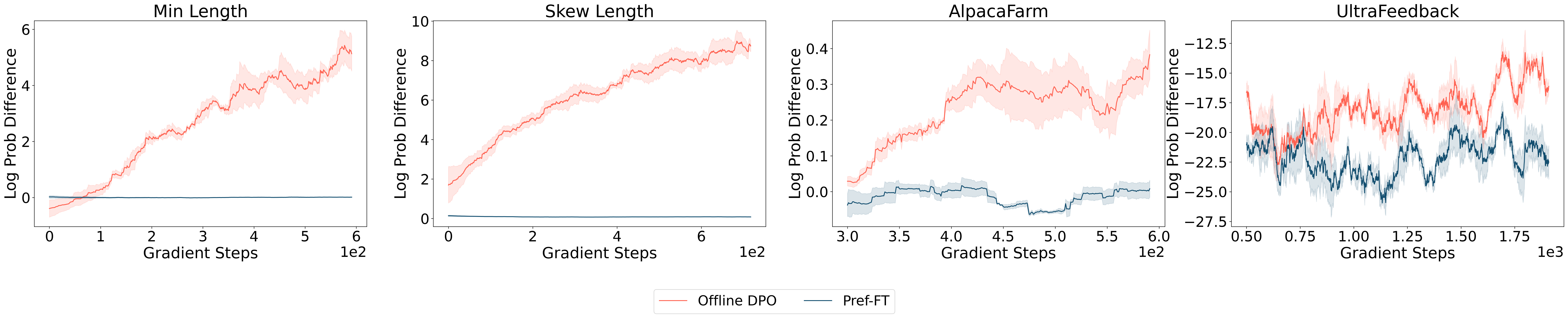}
    \vspace{-0.35cm}
\caption{\label{fig:negative_gradient_log_prob_difference}\footnotesize{\textbf{Difference in likelihoods of preferred and dispreferred responses.} DPO increases the log probability margin $\log \pi_\theta(\by_w|\bx) - \log \pi_\theta(\by_l | \bx)$ more compared to non-contrastive methods such as Pref-FT.}}
    \vspace{-0.25cm}
\end{figure}

\begin{wrapfigure}{r}{0.6\textwidth}
    \vspace{-0.2cm}
    \centering
    \includegraphics[width=0.99\linewidth,height=0.2\textheight]{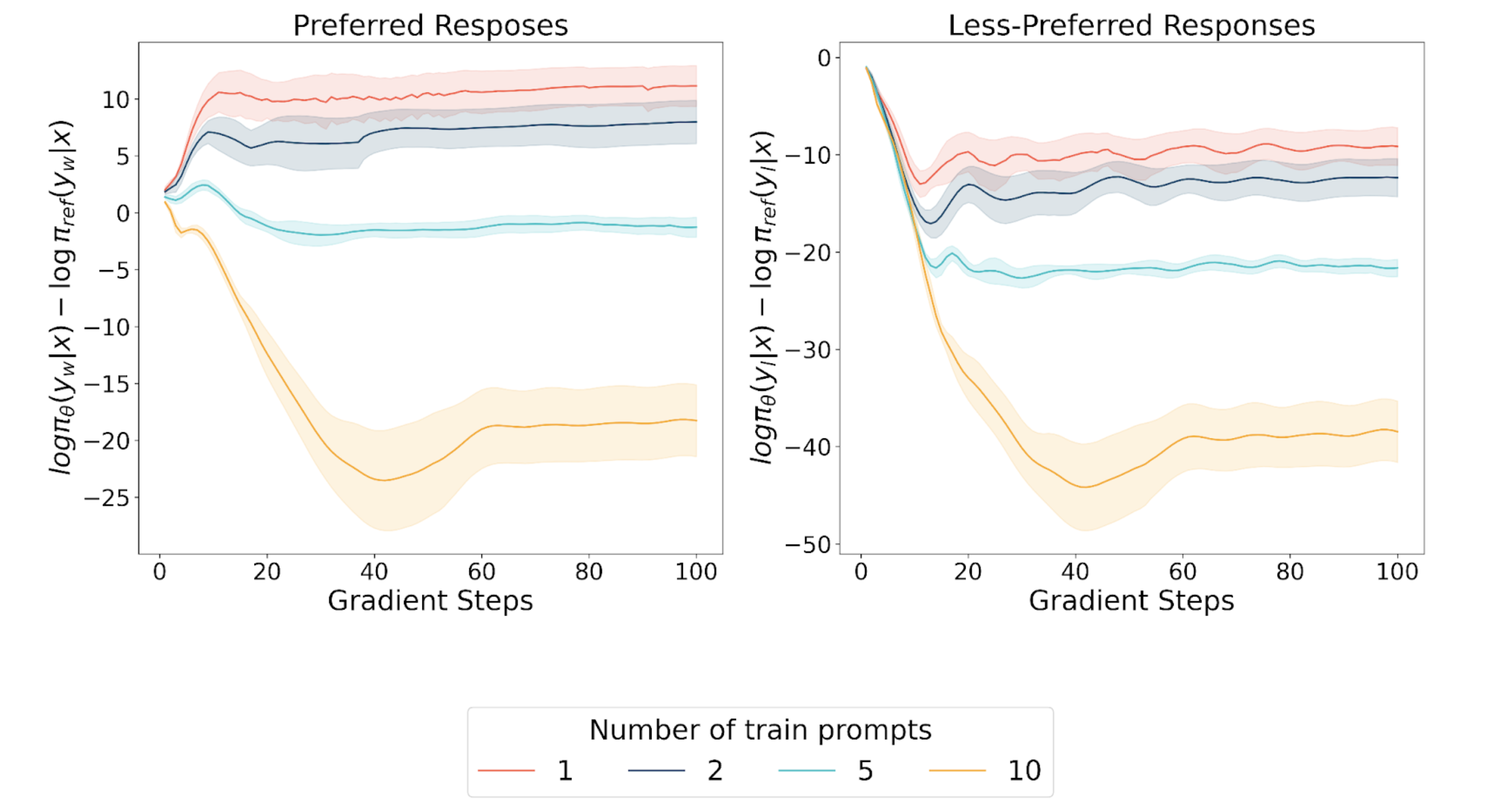}
    \vspace{-0.4cm}    
\caption{\label{fig:transformer_negative_gradients_dataset_size_ablation}\footnotesize{\textbf{DPO implicit reward during training.} We observe that with fewer prompts, contrastive methods can increase the implicit reward, $r_\theta(\bx, \by) = \log\left(\pi_\theta(\by|\bx)\right) - \log \piref(\by|\bx)$, of the preferred response while reducing this quantity for the dispreferred response, however as the number of data points grows, this may not be possible and the likelihood of both positives and negatives might reduce.}}
    \vspace{-0.55cm}
\end{wrapfigure}
\textbf{Changes in log likelihoods depend on model capacity, reference initialization, data size, and composition.} The natural next question is if DPO-like objectives use the probability mass recovered by increasing the reward margin between $\by_w$ and $\by_l$ to increase the probability mass on the preferred responses\footnote{Concurrent work~\citep{rafailov2024from} also studies the induced rewards for DPO and shows that when $\piref(\cdot|\bx)$ is \textbf{exactly} equal to the \textbf{\emph{empirical} distribution} of preferred responses $p(\by_w|\bx)$ in the dataset, then induced rewards will always decrease. This {does not contradict} our findings because this condition is not satisfied in typical fine-tuning pipelines where \emph{both} $\by_w$ and $\by_l$ are sampled from $\piref$. Furthermore, even if $\piref$ is obtained by first running supervised Pref-FT only on $\by_w$, it is unclear whether the parametric model representing $\piref(\cdot|\bx)$ will induce an identical probability distribution to the empirical distribution of preferred responses. That said, it is indeed the case that the likelihood of $\by_w$ decreases often when training with DPO even though the reference policy does not satisfy the condition highlighted in this concurrent work, implying that this phenomenon is a result of many factors (data size, similarity of $\by_w$ and $\by_l$, capacity). We also show in \cref{sec:theory} that with appropriate negatives, likelihoods might not decrease for some contrastive methods.}. We track the induced rewards $\log \pi_\theta(\by_w|\bx) - \log \piref(\by_w|\bx)$ and $\log \pi_\theta(\by_l|\bx) - \log \piref(\by_l|\bx)$ in expectation over prompts $\bx$ on the bandit problem while varying the size of the preference dataset. Following standard protocols, both $\by_l$ and $\by_w$ are sampled from $\piref$.  Observe in \cref{fig:transformer_negative_gradients_dataset_size_ablation} that when the dataset size is small relative to the model capacity, contrastive training via IPO can increase the likelihood of $\by_w$ while reducing the likelihood of $\by_l$. However, as the number of prompts increases, contrastive training counter-intuitively results in a decreasing value of $\log \pi_\theta(\by_w|\bx) - \log \piref(\by_w|\bx)$, although the loss attempts to push up this likelihood term. The recovered probability mass is instead used to increase the likelihood of other out-of-distribution responses. Thus, depending upon $\piref$, dataset size, and composition, contrastive objectives such as DPO extrapolate, and this extrapolation might produce good or bad responses.

We also observe a similar trend in full-scale LLM experiments in~\cref{fig:negative_gradient_reward_vs_gradient_steps}: we observe a decrease in the log-likelihoods of both the preferred and dispreferred responses throughout training on AlpacaFarm with small 1.4B Pythia policies. However, using a Mistral7B model to train a policy on the UltraFeedback dataset results in an increasing value of log-likelihood of $\pi_\theta(\by_w|\bx)$ and a decreasing value of $\pi_\theta(\by_l|\bx)$ when starting from an SFT model on the Ultrachat-200K dataset (same setup as Zephyr~\citep{tunstall2023zephyr}). We believe that these opposite trends are a consequence of the responses in that the UltraFeedback dataset are more semantically distinct from each other, as different responses come from models with different capabilities (e.g., a GPT-4 response is paired with a GPT-3.5 response) such that given enough model capacity, contrastive training can push up likelihoods of $\pi_\theta(\by_w|\bx)$ while pushing down $\pi_\theta(\by_l|\bx)$. In contrast, running Pref-FT increases the likelihoods of both $\by_w$ and $\by_l$ (\cref{fig:negative_gradient_reward_vs_gradient_steps}) despite training only on $\by_w$: this observation about Pref-FT was also noted by concurrent work such as~\citet{hong2024orpo,pang2024iterative}.

\begin{figure}[h!]
\vspace{-0.2cm}
    \centering
    \includegraphics[width=0.99\columnwidth]{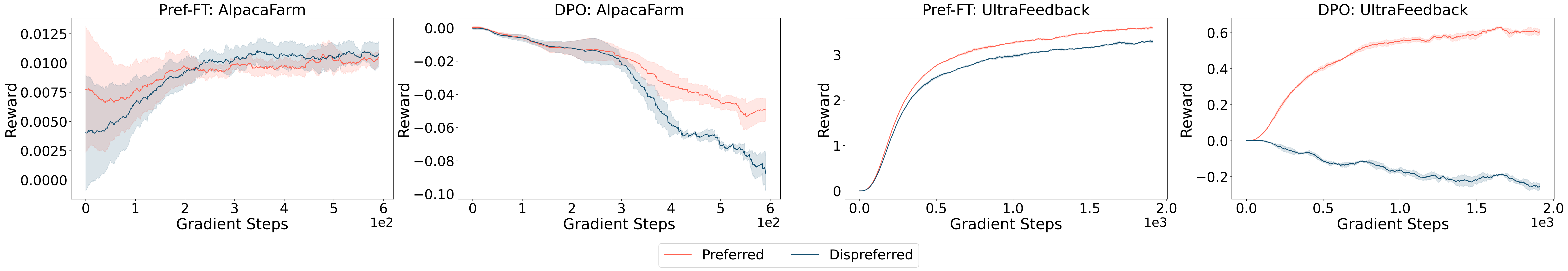}
    \vspace{-0.3cm}
\caption{\label{fig:negative_gradient_reward_vs_gradient_steps}\footnotesize{\textbf{DPO reward estimates for Pref-FT and DPO on AlpacaFarm and UltraFeedback.} For a Pythia-1.4B model trained on AlpacaFarm, DPO decreases the implicit reward, $r_\theta(\bx, \by) = \beta\left[\log \pi_\theta(\by|\bx) - \log \piref(\by|\bx)\right]$, for both $\by_w$ and $\by_l$, whereas Pref-FT increases both. For a Mistral-7B model trained on UltraFeedback, DPO is able to increase the reward for $\by_w$ and decrease the reward for $\by_l$, whereas Pref-FT increases both. In both cases, DPO leads to a higher margin than Pref-FT. }}
    \vspace{-0.25cm}
\end{figure}

\begin{AIbox}{Takeaways for negative gradients}
A negative gradient improves over offline supervised methods when the peak in the reward appears in less likely regions of $\piref$. It can increase the likelihood of $\by_w$ when $\by_l$ is sufficiently different from $\by_w$, model capacity is large, and $\piref$ is chosen appropriately. If not, the margin $\log \pi_\theta(\by_w|\bx) - \log \pi_\theta(\by_l|\bx)$ will still be larger when a negative gradient is used, but the recovered probability mass will go into increasing likelihoods of other responses, not $\by_w$.
\end{AIbox}

\vspace{-0.2cm}
\subsection{{Question 3: } On-Policy Sampling and Negative Gradients are Complementary}
\label{sec:complementarity}
\vspace{-0.1cm}
Based on our findings that both on-policy sampling and negative gradients are independently effective, we now study if combining them would provide any additional benefits. To understand this, we empirically study a straightforward on-policy variant of DPO/IPO: instead of utilizing the PPO or Best-of-N objective on on-policy samples, for each prompt $\bx$, we sample $N$ responses from the policy $\by_1, \ldots, \by_n \sim \pi_\theta(.|\bx)$, rank them according to a reward model $r_\phi$, and construct preference pairs by taking the higher reward completion as the preferred one and lower reward completion as the dispreferred one. This recipe is similar to concurrent works such as \citet{rosset2024direct}. Then we calculate the DPO/IPO loss on this preference dataset and update our model accordingly.  

\begin{figure}[htbp]
    \centering
    \includegraphics[width=0.99\columnwidth,height=0.2\textheight]{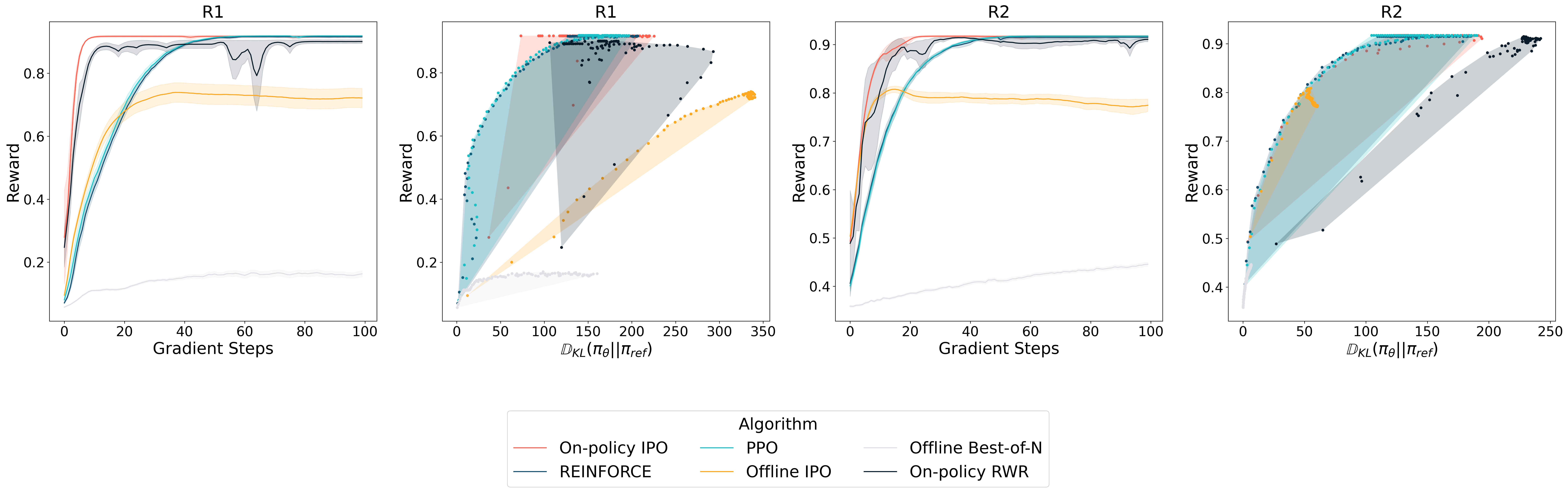}
    \vspace{-0.3cm}
    \caption{\label{fig:transformer_complimentary}\footnotesize{\textbf{On-policy sampling + negative gradients in bandit setup.} Complimentary benefit of on-policy sampling and negative gradients. Online IPO (using both on-policy sampling and negative gradients) performs better than offline IPO (negative gradients but no on-policy sampling) and RWR (on-policy sampling but no negative gradients).}}
    \vspace{-0.2cm}
\end{figure}

\textbf{Performance on bandit and synthetic LLM problems.} \cref{fig:transformer_complimentary} shows that the on-policy version of IPO achieves both faster convergence and better performance compared to the offline version, for both $\mathbf{R}_1$ and $\mathbf{R}_2$ in the didactic bandit problem. We also ran on-policy DPO in synthetic LLM problems we studied and found it to converge significantly faster and to a better solution than offline DPO, on-policy RL, and on-policy variants of supervised learning approaches as shown in \cref{fig:complimentarity_length}. We also find that on-policy versions of contrastive approaches exhibit favorable computational vs wall-clock time tradeoffs compared to purely on-policy RL methods and even offline contrastive methods that may not find as good solutions as their on-policy counterparts (see \cref{app:time_tradeoff}).

\begin{figure}[h!]
\vspace{-0.1cm}
    \centering
    \includegraphics[width=0.85\columnwidth]{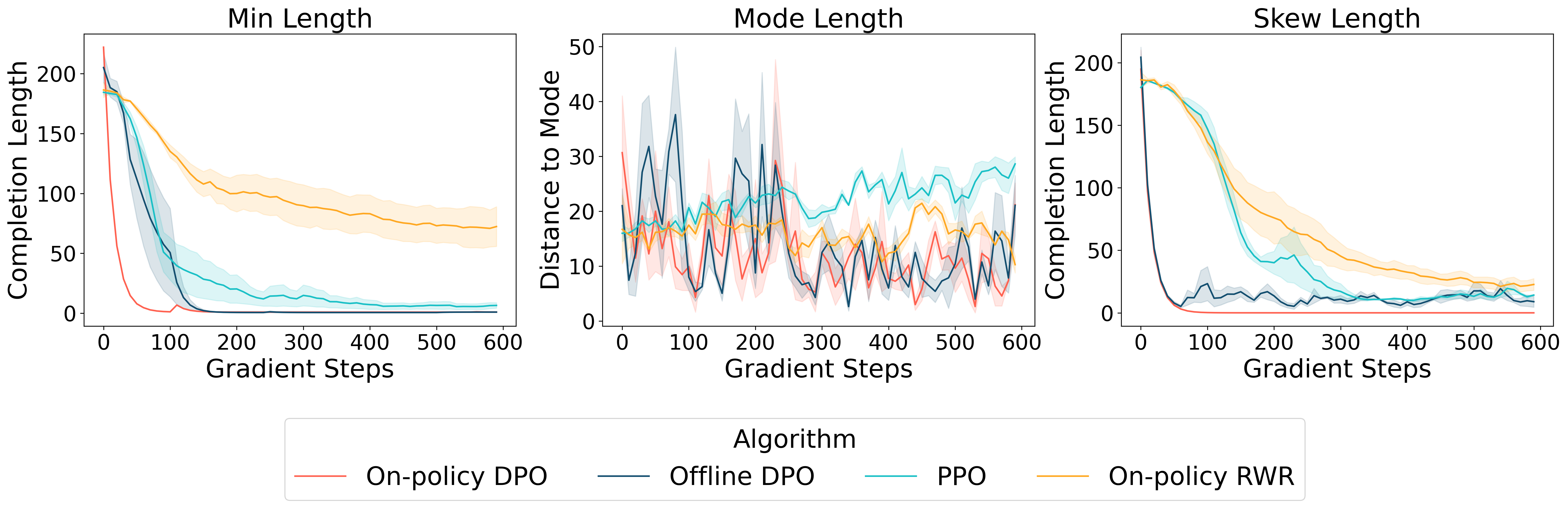}
    \vspace{-0.3cm}    \caption{\footnotesize{
    \textbf{On-policy sampling + negative gradients in LLM length experiments.} Complimentary benefit of on-policy sampling and negative gradients on the synthetic LLM length experiments. On-policy DPO performs the best where optimal policy and reference policy lies far from each other (min length and skew length), and all algorithms perform similarly when these two policies are close (mode length).}}
    \vspace{-0.3cm}
    \label{fig:complimentarity_length}
\end{figure}

\textbf{Why can on-policy versions of contrastive methods perform better than on-policy RL?} We saw in \cref{sec:negative_gradient} that offline contrastive training with a negative gradient was effective at quickly reorganizing probability mass to high-reward responses covered by the preference data. When combined with on-policy sampling, this behavior results in faster convergence: for any given batch of on-policy data, contrastive training with a negative gradient can quickly reconfigure the policy distribution within the support of the on-policy data obtained thus far (i.e., it provides a stronger, low-variance learning signal). Similarly to how best-of-N + negative gradient outperforms vanilla best-of-N but underperforms DPO in \cref{fig:transformer_negative_gradients}, PPO also improves over RWR without a negative gradient term (in the bandit setting this corresponds to a better reward-KL tradeoff in \cref{fig:transformer_complimentary} and in the synthetic LLM setting this appears in final performance), but it is still unable to match on-policy DPO in \cref{fig:complimentarity_length}. Note that this does not mean that on-policy DPO would always outperform PPO, but that it might be a good choice for users to experiment with on-policy versions of contrastive methods. 

\begin{AIbox}{Takeaways for on-policy sampling + negative gradient}
On-policy sampling and offline negative gradients present complementary benefits, in that the best offline loss function with negative gradients can be used to train on on-policy data, improving over on-policy RL or supervised learning. Conceptually, while sampling responses on policy provides coverage of the response space, an effective negative gradient loss provides a stronger learning signal given a set of samples. It can also result in computational benefits.
\end{AIbox}

\vspace{-0.2cm}
\section{Conceptual Unification and Theoretical Analysis}
\vspace{-0.2cm}
\label{sec:theory}

With empirical results showing the benefits of on-policy sampling and negative gradient for preference fine-tuning of LLMs, in this section, we attempt to conceptually understand the benefits by building a mental model. In this section, we will first unify these seemingly distinct notions of on-policy sampling and negative gradient into a unified notion of mode-seeking objectives, and contrast them against mode-covering maximum likelihood objectives. Then, we will contrast the learning dynamics of the reverse KL-divergence, a representative mode-seeking objective against the mode-seeking forward KL-divergence (i.e., the supervised learning loss) to intuitively explain some of our findings.

\vspace{-0.15cm}
\subsection{Seeking Modes Unifies On-Policy Sampling and Negative Gradients}
\label{sec:unification}
\vspace{-0.1cm}

In this section, we will show that the notion of mode-seeking divergences unifies on-policy sampling and negative gradients for the various objectives we investigated in the paper. Specifically, we show below that several on-policy RL methods that we studied optimize the reverse KL-divergence, and are hence mode-seeking, offline contrastive methods that employ a negative gradient are also mode-seeking, and finally, supervised weighted maximum likelihood approaches (e.g., offline Best-of-N, Pref-FT, Binary FeedMe) are mode-covering. 
First, we show that on-policy sampling leads to mode-seeking behavior. To do this, we prove that RL and supervised objectives combined on-policy sampling optimize the reverse KL divergence, which is known to be mode-seeking.

\begin{lemma}
\label{lemma:on_policy}
    On-policy RL and on-policy weighted-likelihood methods optimize a regularized version of a reverse KL-divergence with respect to the optimal policy and are hence mode seeking.
\end{lemma}
\vspace{-0.2cm}

A proof for Lemma~\ref{lemma:on_policy} is shown in \cref{appendix:on_policy_mode_seeking}. Next, we show that offline contrastive methods that employ a negative gradient are also mode-seeking. While these approaches do not optimize the reverse KL-divergence, we can still show that the probability mass obtained by minimizing density on negative responses $\by_l$ gets disproportionately utilized, far more for increasing the probability mass on the ``mode'' (i.e., highest probability categories under the current policy $\pi_\theta$) compared to other categories. When the offline dataset consists of multiple high-reward categories, this preference to put more probability mass on the mode of the current policy results in mode-seeking behavior, compared to increasing probability mass on all high-reward categories.

\begin{lemma}
\label{lemma:negative_grad}
    Let $\theta_t$ denote the parameters of the model at a given iteration $t$.  
    Consider contrastive approaches that induce a negative gradient under a functional form shown below:
    \begin{align}\label{eq:update_negative_gradient}
        \theta_{t+1} \leftarrow \theta_t + \eta~ \mathbb{E}_{\bx, \by_w, \by_l \sim \mathcal{D}} \left[\nabla_\theta \log \pi_\theta(\by_w|\bx) \cdot c_1(\bx, \by_w, \by_l) - \nabla_\theta \log \pi_\theta(\by_l|\bx) \cdot c_2(\bx, \by_w, \by_l)  \right] \Big\vert_{\theta_t},
    \end{align}
    where $c_1$ and $c_2$ are non-negative functions that depend on the reward value and the associated samples, $\by_w$ and $\by_l$. In contrast, weighted maximum likelihood without the negative gradient sets $c_2 = 0$. Define $\omega_{t} := \log \pi_\theta(\by_w|\bx) - \log \pi_\theta(\by_l|\bx)$. Then, for all models $\theta$ and for all $t$, there always exists an appropriate dataset of positive and negative samples $\mathcal{D}$, such that:
    \begin{align}
        \mathbb{E}_{\bx, \by_w, \by_l \sim \mathcal{D}} \left[\omega_{t + 1} \right] \Big\vert_{{c_2 > 0}} \geq \mathbb{E}_{\bx, \by_w, \by_l \sim \mathcal{D}} \left[\omega_{t + 1} \right] \Big\vert_{{c_2 = 0}}.
    \end{align}
    In addition, if the model class $\pi_\theta$ and $\by_l$ can jointly realize the following gradient alignment condition (note that for any $\theta_t$, there always exists a $\by_l$ that satisfies this condition): 
    \begin{align*}
        \forall \theta \in \left[\theta_t \right]_t, \mathbb{E}_{\bx, \by_w, \by_l \sim \mathcal{D}} \left[ \nabla_\theta \log \pi_{\theta}(\by_w|\bx)^\top \nabla_\theta \log \pi_{\theta}(\by_l|\bx)  \right] \leq 0,
    \end{align*} then, we find that the likelihood of positives is larger (and similarly likelihood of negatives is smaller) when $c_2 > 0$, i.e., when a negative gradient term is used:
    \begin{align}
        \mathbb{E}_{\bx, \by_w, \by_l \sim \mathcal{D}} \left[\log \pi_\theta(\by_w|\bx) \right] \Big\vert_{c_2 > 0, \theta = \theta_t} ~& \geq \mathbb{E}_{\bx, \by_w, \by_l \sim \mathcal{D}} \left[\log \pi_\theta(\by_w|\bx) \right] \Big\vert_{c_2 = 0, \theta = \theta_t} \\
        \mathbb{E}_{\bx, \by_w, \by_l \sim \mathcal{D}} \left[\log \pi_\theta(\by_l|\bx) \right] \Big\vert_{c_2 > 0, \theta = \theta_t} ~& \leq \mathbb{E}_{\bx, \by_w, \by_l \sim \mathcal{D}} \left[\log \pi_\theta(\by_l|\bx) \right] \Big\vert_{c_2 = 0, \theta = \theta_t}
    \end{align}
\end{lemma}
\vspace{-0.2cm}

A proof for \cref{lemma:negative_grad} is provided in \cref{appendix:dpo_mode_seeking}. This result indicates that for appropriate negative responses, a contrastive update accelerates the rate of increase of probability mass on $\by_w$, for any model class $\pi_\theta$ and reference initialization $\theta_0$, compared to setting $c_2 = 0$, which offline weighted maximum likelihood. This corresponds to mode-seeking behavior. The update induced by DPO admits a similar form (see the discussion after Equation 7 in \citet{rafailov2023direct}). This theoretical result also corroborates our findings in the experiments in Section~\ref{sec:negative_gradient} regarding the negative gradient term. The gradient of IPO also admits a similar form (\cref{appendix:dpo_mode_seeking}). 

Next, we note that purely offline versions of supervised methods such as RWR, ReST, and BoN, that only maximize weighted likelihood are mode-covering because these objectives can be shown to maximize the forward KL-divergence against the optimal policy (proof in \cref{appendix:supervised_offline}).
\begin{lemma}
\label{lemma:supervised_offline}
    Consider offline supervised methods that maximize weighted log-likelihood:
    \begin{align}
        \mathcal{L}_{\text{off-sup}}(\pi_\theta; \piref) = - \mathbb{E}_{\bx \sim \mathcal{D}_\text{pref}}\left[ \mathbb{E}_{\by \sim \piref(.|\bx)}[\log \pi_\theta(\by|\bx) \cdot F(\bx, \by)]\right]
    \end{align}
    where $F(\bx, \by) \geq 0$ is the weight for $(\bx, \by)$. Furthermore, $\sum_{\by} F(\bx, \by) > 0$ (i.e., for every $\bx$, there exists a response $\by$ with non-zero $F(\bx, \by)$). Then these methods optimize a forward KL-divergence. 
\end{lemma}
\vspace{-0.2cm}

\vspace{-0.2cm}
\subsection{Case Study: Mode-Seeking Reverse KL vs. Mode-Covering Forward KL}
\label{sec:reverse_vs_forward}
\vspace{-0.1cm}

Having seen that mode-seeking and mode-covering divergences can unify on-policy sampling and negative gradients, in this section, we perform a theoretical analysis to quantify the behavior of the two representative mode-seeking and mode-covering objectives: reverse KL (mode-seeking) and forward KL (mode-covering) objectives on categorical distributions, parameterized via independent logits. Our goal is to formalize the intuition that a mode-seeking objective can sharpen the probability mass on only certain high-reward regions, thereby leading to aggressive reorganization of probability mass. This helps corroborate our experiments that on-policy sampling in a reward model and offline negative sampling is still useful to quickly align the policy with the target distribution.

\textbf{Notation and setup.} For this result, we will study training a categorical distribution $p(\bx)$ to match the theoretically optimal fine-tuned policy, $q(\bx)$. We assume that $p(\bx) \propto \exp(f(\bx))$, where each logit $f(\bx)$ is an independent parameter. We train $p(\bx)$ by performing gradient descent, starting from an initial reference distribution $p_0$ on a fine-tuning loss with gradient descent and a learning rate $\eta$. We denote the distribution at step $t$ of this gradient descent as $p_t$. For this analysis it would be helpful to explicitly write out the parameter updates at any iteration $t$, induced by forward and reverse KL.

\begin{lemma}\label{lemma:gradient_characterization_forward_vs_reverse_kl}
    For any given distribution $p_t$, with $p_t(\bx) = \exp (f_t(\bx))$, the updates induced by the forward and reverse KL-divergences within one step of gradient descent with a learning rate $\eta$ are given by:
    \begin{align}
        \text{Forward KL:}~~~~~& \log \frac{p^f_{t+1}(\bx)}{p_t(\bx)} =  \eta \left( q(\bx) - p_t(\bx) \right) + \mathbb{Z}. \\
        \text{Reverse KL:}~~~~~& \log \frac{p^r_{t+1}(\bx)}{p_t(\bx)} = \eta \left( p_t(\bx) \left[ \log \frac{q(\bx)}{p_t(\bx)} + \mathbb{D}_{\text{KL}}(p_t(\cdot) || q(\cdot))\right] \right) + \mathbb{Z}',
    \end{align}
    where $\mathbb{Z}$ and $\mathbb{Z}'$ denote constant normalization factors.
\end{lemma}
\vspace{-0.2cm}

For a proof of~\cref{lemma:gradient_characterization_forward_vs_reverse_kl}, see~\cref{appendix:forward_reverse_kl_gradient_characterization}. 
In principle, upon convergence, both the reverse and forward KL-divergences should find the optimally fine-tuned distribution, $q(\bx)$ in this simple setting. But to understand their behavior in relevant practical situations, we are particularly interested in understanding their behavior at intermediate points during training, when either divergence is not minimized to exactly 0. Insights about intermediate points in training can make useful predictions about practical problems when early stopping is used to prevent overfitting and the loss is rarely 0. Thus, our result below attempts to characterize these objectives at any given iteration $t$:

\begin{theorem}
\label{thm:commital}
  Let $p^f_{t+1}(\bx)$ be the distribution obtained after one gradient step, starting from $p_t$ using the forward KL divergence. Likewise, let $p^r_{t+1}(\bx)$ be the distribution obtained using the reverse KL divergence, from $p_t$. Define $\Delta^f_t$ and $\Delta^r_t$ as the difference of log probability ratios across two categories $\bx_1$ and $\bx_2$, obtained from the forward and reverse divergences respectively:
  \begin{align}
      \Delta_t^f(\bx_1, \bx_2) := \log \frac{p^f_{t+1}(\bx_1)}{p_t(\bx_1)} - \log \frac{p^f_{t+1}(\bx_2)}{p_t(\bx_2)},
  \end{align}
  and $\Delta_t^r$ is similarly defined. Then we have the following (for appropriate positive constants $\beta$, $\delta_1$, $\delta_2)$:
  \begin{enumerate}
      \item \textbf{Reverse KL modifies probability mass more aggressively than the forward KL.} If $\bx_1$ and $\bx_2$ are such that, $\delta_1 \leq p_t(\bx_1) = p_t(\bx_2) \leq 1 - \delta_2$ (where $\delta_1 > 0$, $\delta_2 > 0$), but $q(\bx_1) \geq q(\bx_2) + \beta$, then, $\Delta^r_t(\bx_1, \bx_2) > \Delta^f_t(\bx_1, \bx_2)$.
      \item \textbf{Reverse KL increases probability mass only on a subset of categories that equal target likelihoods.} If $\bx_1$ and $\bx_2$ are such that, $p_t(\bx_2) + \beta \leq  p_t(\bx_1) \leq 1 - \delta_2$, and $q(\bx_1) = q(\bx_2) > \mathbf{c}_0 \cdot p_t(\bx_1)$, where $\mathbf{c}_0$ is a positive constant $> 1$, then,  $\Delta^r_t(\bx_1, \bx_2) > \Delta^f_t(\bx_1, \bx_2)$.
      \item \textbf{Reverse KL aggressively reduces probability mass on less-likely categories in the target distribution.} If $\bx_1$ and $\bx_2$ are such that, $p_t(\bx_2) + \beta \leq  p_t(\bx_1) \leq 1 - \delta_2$, and $q(\bx_1) = q(\bx_2) < \mathbf{c}_1 \cdot p_t(\bx_2)$, where $\mathbf{c}_1$ is a positive constant $< 1$, then,  $\Delta^r_t(\bx_1, \bx_2) < \Delta^f_t(\bx_1, \bx_2)$.
  \end{enumerate}
\end{theorem}

A proof of Theorem~\ref{thm:commital} is shown in~\cref{appendix:quantifying_dufference}. Essentially, this theorem enlists several cases where the forward KL modifies probability mass in different amounts across various categories, but the reverse KL acts disproportionately. In particular, case 1 says that the \textbf{reverse KL exhibits more disproportionate probability mass changes on categories} with equal likelihood $p_t(\bx)$, due to the logarithmic dependency on the probability mass $q(\bx)$ (compared to the linear dependency for the forward KL). Case 2 says that when the target value $q(\bx)$ for two categories is much larger than the probability mass currently assigned to those categories, then the reverse KL can attempt to preferentially increase probability mass more in the category with a larger likelihood $p_t(\bx)$ under certain conditions. Finally, case 3 shows that when the likelihood of a category is significantly larger than the target $q(\bx)$, the reverse KL is more effective at reducing this probability mass and re-distributing it to other categories within one update step. 
Finally, consider another special case, where the difference $q(\bx) - p_t(\bx)$ is identical for two categories $\bx_1$ and $\bx_2$. In this case, while the forward KL will increase log probability ratios for both $\bx_1$ and $\bx_2$ equally, i.e., $\Delta^f(\bx_1, \bx_2) = 0$, the reverse KL will prioritize the category with a higher $p_t(\bx)$ value. These results highlight some scenarios under which the reverse KL can more efficiently re-organize probability mass across categories. 

\begin{AIbox}{Mode-seeking vs. mode-covering objectives for categorical distributions}
Typically the benefits of mode-seeking behavior are more apparent when the model $p(\bx)$ is unable to realize the target distribution $q(\bx)$ such that minimizing either KL would give rise to different solutions. Unlike this argument, we show that even when the $p(\bx)$ can fully represent the target distribution $q(\bx)$, reverse KL can quickly re-distribute probability mass to only a subset of the required categories likely in target distribution, within a few gradient steps. 
\end{AIbox}

\vspace{-0.2cm}
\section{Discussion, Conclusion, and Limitations}
\vspace{-0.15cm}

We attempted to understand which components are particularly important for fine-tuning language models with preference data. Through extensive experiments on different fine-tuning problems in both didactic and LLM settings, we established that on-policy sampling is crucial for good performance especially when the peak in the ground-truth reward lies in less-likely regions of the reference policy initialization. That said, in practice, doing so requires preference datasets with broader coverage than the reference policy. We also showed that negative gradients can enable faster convergence and that objectives that induce a negative gradient are complementary to using on-policy sampling. Finally, we show that the notion of mode-seeking divergences unifies the notion of on-policy sampling and negative gradient. Our case study comparing forward and reverse KL divergences demonstrates the superiority of the reverse KL divergence in re-distributing probability mass efficiently, supporting our empirical findings pertaining to on-policy sampling and negative gradients. 

While we conceptualize our observations, a limitation is that we don't derive rigorous statistical guarantees in this work. As an example, we note that while the notion of concentrability coefficients (and associated guarantees) can potentially provide guarantees on on-policy sampling, to the best of our knowledge the notion of the negative gradient is not fully studied in the literature. We conjecture that negative gradient can perhaps be formalized statistically from the lens of providing a lower variance learning signal; it would be interesting for future work to formalize this. It would also be interesting to study more recent approaches based on minimax formulations (e.g.,~\citet{2023arXiv231200886M,2024arXiv240110020Y,swamy2024minimaximalist,2024arXiv240101335C}) in our empirical and conceptual framework. Next, while we consider the coverage of preference data relative to that of the reference policy in our study, this is a simplification that does not account for the coverage of the pre-training distribution which future work can incorporate. Finally, we remark that our study does not explore the effect of reward model quality, which tends to also play a central role in LLM fine-tuning. It would be interesting to extend our analysis to incorporate the role of reward model quality and parameterization.

\vspace{-0.2cm}
\section*{Acknowledgements}
\vspace{-0.2cm}

We would like to thank Yi Su, Rishabh Agarwal, Zhang-Wei Hong, Young Geng, Abitha Thankaraj, Yuxiao Qu, So Yeon Min, Yutong He, Kevin Li, Sukjun Hwang, Khurram Yamin, Charlie Snell, Amrith Setlur, Kaylee Burns, Eric Mitchell, and others in CMU Russ Lab, CMU Auton Lab, Stanford IRIS Lab, and Stanford Ermon Group for discussions and feedback. AK thanks Aleksandra Faust, George Tucker, and Sergey Levine for informative discussions. This research is supported by computational resources from Google TPU Research Cloud (TRC) and the National Science Foundation. FT thanks Ruslan Salakhutdinov for insightful suggestions during this project. AS gratefully acknowledges the support of the NSF Graduate Research Fellowship Program.

\bibliography{main}

\newpage
\appendix
\onecolumn

\part*{Appendices}

\section{Connections to Existing Fine-Tuning Results}
\label{sec:existing}

Our proposed framework also allows us to explain experiments and evaluations in several existing LLM fine-tuning results, and as a result, implies several practical guidelines for LLM practitioners. On the AlpacaFarm benchmark~\citep{dubois2024alpacafarm}, our results corroborate the gap between conditional supervised fine-tuning objectives such as binary FeedME and reward conditioning, and RL or contrastive training methods such as PPO and DPO: these results are perhaps even more extreme in that these conditional and weighted supervised fine-tuning objectives are not even able to outperform regular SFT. Methods that utilize on-policy sampling such as ReST~\citep{gulcehre2023reinforced} and Quark~\citep{lu2022quark} do outperform SFT but still underperform on-policy RL or on-policy contrastive training.
The top-performing methods on the benchmark are offline DPO, which uses a negative gradient, and PPO, which leverages on-policy sampling.

Additionally, methods such as self-rewarding language models~\citep{yuan2024self}, RSO~\citep{liu2024statistical}, OAIF~\citep{2024arXiv240204792G}, DR-PO~\citep{2024arXiv240408495C}, Hybrid-DPO~\citep{2023arXiv231211456X}, and RS-DPO~\citep{2024arXiv240210038K} couple on-policy sampling or rejection sampling with contrastive training objectives. These works corroborate our observation regarding the efficacy of on-policy sampling and negative gradients and how they are complementary. Approaches such as CRINGE~\citep{2022arXiv221105826A} combine maximum likelihood with a token level contrastive loss term and show gains over solely utilizing supervised likelihood, corroborating our insights about negative gradients. 

Concurrently to us, \citet{xu2024isdpo} show that on many practical LLM fine-tuning problems offline DPO underperforms on-policy PPO. While we do not study the same LLM fine-tuning problems, the insights from this work corroborate our findings, which in turn extend insights from this work. For instance, this work observes that DPO can learn to find out-of-distribution responses, which is consistent with our analysis in Section~\ref{sec:mechanisms} that offline DPO training might increase probability mass on the highly likely regions of $\pi_\theta$, deviating significantly from the distribution of preferred responses $p(\by_w|\bx)$. To avoid this issue, this work prescribes an iterated DPO recipe where the reference policy (i.e., the SFT policy in their setting) is used to iteratively collect new samples for DPO training. Section~\ref{sec:complementarity} arrives at a similar conclusion that using on-policy samples for policy optimization, though we recommend collecting samples from the current policy and not the reference policy, which might fail to cover important regions of the space when the peak in the reward function appears farther away from the high-likely regions of the reference policy.

\vspace{-0.2cm}
\section{Computational vs Wall-Clock Time Tradeoff for Various Methods}
\label{app:time_tradeoff}
\vspace{-0.1cm}

\begin{table}[H]
\centering
\resizebox{\textwidth}{!}{%
    \begin{tabular}{lcccccc}
    \hline
     & \multicolumn{2}{c}{\textbf{Bandit (R1)}} & \multicolumn{2}{c}{\textbf{Min Length}} & \multicolumn{2}{c}{\textbf{Skew Length}} \\
     & \textbf{Reward ($\uparrow$)} & \textbf{Time} & \textbf{Completion Length ($\downarrow$)} & \textbf{Time} & \textbf{Completion Length ($\downarrow$)} & \textbf{Time} \\ \hline
    \textbf{Offline DPO / IPO} & 0.82 (0.04) & 1.7 hours & 1.0 (0.0) & 1.3 hours & 11.8 (14.0) & 0.12 hours \\
    \textbf{On-policy PPO} & 0.92 (0.01) & 0.93 hours & 20.5 (25.4) & 4.84 hours & 15.8 (11.1) & 7.26 hours \\
    \textbf{On-policy RWR} & 0.88 (0.01) & 0.12 hours & 65.5 (36.7) & 15.5 hours & 15.8 (9.3) & 15.5 hours \\ 
    \textbf{On-policy DPO / IPO} & 0.92 (0.01) & 0.12 hours & 1.0 (0.0) & 0.4 hours & 0.0 (0.0) & 0.4 hours\\
    \hline
    \end{tabular}
}
\vspace{-0.2cm}
\caption{\footnotesize{\textbf{Wall-clock time comparisons.} Comparison between on-policy and offline variants of contrastive objectives (DPO/IPO) in terms of reward and wall-clock time required till convergence of the run. Generally, on-policy contrastive approaches achieve both superior reward and wall-clock time as opposed to offline contrastive approaches (offline DPO/IPO) and on-policy RL (PPO, RWR). Synthetic LLM experiments use a single A40 GPU. Bandit experiments use a Intel(R) Xeon(R) CPU E5-2698 v4 @ 2.20GHz CPU, with 4 threads.}}
\label{tab:on_policy_dpo_wall_clock_time}
\vspace{-0.4cm}
\end{table}

A natural takeaway extending the empirical results from Section~\ref{sec:complementarity} is that on-policy variants of contrastive approaches might provide for an better tradeoff between computation and wall-clock time. We perform a comparison of wall-clock time needed to run our experiments in \cref{tab:on_policy_dpo_wall_clock_time}. in particular, we found that on-policy DPO only requires 0.4 hours to converge, while offline DPO requires a wall-clock time of 1.3 hours to converge to the same solution in the \textbf{Min Length} scenario. In the \textbf{Skew Length} scenario, where the learned policy must deviate from the initial reference policy substantially, we find that while offline DPO can converge a bit quickly (0.12 hours), it flatlines at a sub-optimal solution (completion length of 11.8) as compared to on-policy DPO which takes merely 0.4 hours to reach a more optimal solution. This is far more time-efficient compared to other on-policy methods such as PPO and RWR that present a sampling bottleneck.

\section{More Details on Conceptual Unification and Theoretical Analysis}
\label{sec:appendix_theory}

\subsection{Unifying On-Policy Sampling and Negative Gradients via Mode-Seeking Divergences} 
\label{appendix:unification}

Here we provide proofs for the claims in \ref{sec:unification}. We will show that on-policy methods and offline constrastive methods, both are mode-seeking as opposed to supervised maximum likelihood approaches, which are mode-covering. This conceptually explains the differences in their behaviors that we observe in our experiments.

\subsubsection{On-policy Methods Are Mode-Seeking} \label{appendix:on_policy_mode_seeking}
First, we prove \cref{lemma:on_policy}, i.e., we want to show that on-policy RL methods and on-policy versions of weighted supervised learning methods optimize regularized version of a reverse KL-divergence. 

\begin{proof} \label{proof:lemma_on_policy}
    Both on-policy RL algorithms and on-policy versions of weighted supervised learning, optimize the following loss function:
    \begin{align}\label{eq:on_policy_rl_objective}
        \mathcal{L}_{\text{RL}}(\mathcal{D}_{\text{pref}}, \pi_{\theta}) = -\mathbb{E}_{\bx \sim \mathcal{D}_{\text{pref}}} [ \mathbb{E}_{\by \sim \pi_{\theta}(.|\bx)}[r(\bx, \by)]  - \beta \mathbb{D}_{\text{KL}}[\pi_{\theta}(.|\bx) || \piref(.|\bx)]]
    \end{align}
    
    Following Appendix A.1 of \citet{rafailov2023direct}, there exists some policy $\pi^*$ such that we can express the reward function $r(\bx, \by)$ as follows:
    \begin{align*}
        r(\bx, \by) = \beta \log Z(\bx) + \beta \log \left(\frac{\pi^*(\by|\bx)}{\piref(\by|\bx)}\right)
    \end{align*}
    where $Z(\bx) = \sum_{\by} \piref(\by|\bx) \exp\left(\frac{r(\bx, \by)}{\beta}\right)$ is the partition function. Combining these two, we get:

    \begin{align*}
        \mathcal{L}_{\text{RL}}(\mathcal{D}_{\text{pref}}, \pi_{\theta}) = {} & - \beta \mathbb{E}_{\bx \sim \mathcal{D}_{\text{pref}}}\left[\mathbb{E}_{\by \sim \pi_{\theta}(.|\bx)}\left[\log Z(\bx) + \log\left(\frac{\pi^*(\by|\bx)}{\piref(\by|\bx)}\right)\right] - \mathbb{D}_{\text{KL}}[\pi_{\theta}(.|\bx) || \piref(.|\bx)]\right] \\
        = {} & - \beta \mathbb{E}_{\bx \sim \mathcal{D}_{\text{pref}}}\left[\mathbb{E}_{\by \sim \pi_{\theta}(.|\bx)}\left[\log Z(\bx) + \log\left(\frac{\pi^*(\by|\bx)}{\piref(\by|\bx)}\right)\right] - \mathbb{E}_{\by \sim \pi_{\theta}(.|\bx)}\left[\log\left(\frac{\pi_{\theta}(\by|\bx)}{\piref(\by|\bx)}\right)\right]\right] \\
        = {} & -\beta \mathbb{E}_{\bx \sim \mathcal{D}_{\text{pref}}}\left[\mathbb{E}_{\by \sim \pi_{\theta}(.|\bx)}\left[\log Z(\bx) - \log\left(\frac{\pi_{\theta}(\by|\bx)}{\pi^*(\by|\bx)}\right)\right]\right] \\
        = {} & -\beta \mathbb{E}_{\bx \sim \mathcal{D}_{\text{pref}}}[\log Z(\bx)] + \beta \mathbb{E}_{\bx \sim \mathcal{D}_{\text{pref}}}[\mathbb{D}_{\text{KL}}[\pi_{\theta}(.|\bx) || \pi^*(.|\bx) ]]
    \end{align*}
    Note that $Z(\bx)$ does not depend on $\pi_{\theta}$. Therefore, minimizing $\mathcal{L}_{\text{RL}}$ with respect to $\pi_{\theta}$ is equivalent to optimizing the reverse KL-divergence. Since optimizing the reverse KL-divergence is mode-seeking, we see that on-policy RL algorithms have mode-seeking behavior. 
\end{proof}

\subsubsection{Contrastive Approaches (e.g., DPO/IPO) are Mode-Seeking} \label{appendix:dpo_mode_seeking}
Next, we show that this is also the case for contrastive approaches as we prove \cref{lemma:negative_grad}.

\begin{proof}
    First consider an input $\bx$. Consider the gradient update (with a small enough learning rate):
    $$\theta_{t + 1} \leftarrow \theta_t - \eta [\nabla_\theta \log \pi_\theta(\by_w | \bx) \cdot c_1(\bx, \by_w, \by_l) - \nabla_\theta \log \pi_\theta(\by_l|\bx) \cdot c_2(\bx, \by_w, \by_l)]$$
    
    We shall prove that for all possible models $\theta$ and for all $t$, there always exists appropriate pairing of positive and negative samples $(\by_w, \by_l)$, such that after taking the gradient update, we have:
    \begin{align*}
        \omega_{t + 1} \Big\vert_{{c_2 > 0}} \geq \omega_{t + 1}  \Big\vert_{{c_2 = 0}}
    \end{align*}
    The core idea behind this proof is the normalization of the probability simplex. We proceed with a combination of mathematical inducation and contradiction: assume that $\omega_t \Big\vert_{{c_2 > 0}} \geq \omega_t \Big\vert_{{c_2 = 0}}$, but for all possible pairings $(\by_w, \by_l)$, we have $\omega_{t + 1} \Big\vert_{{c_2 > 0}} < \omega_{t + 1}  \Big\vert_{{c_2 = 0}}$. We will show that this is not possible. To do this, we first derive the expressions for $\omega_{t+1}$ and then study under what conditions is it possible that for any pairing of positives and negatives, $\omega_{t+1}$ is smaller when $c_2 > 0$. The expression for $\omega_{t+1}$ is given by:
    \begin{align*}
        \omega' &= \omega + \eta \left(\theta' - \theta \right)^\top \left(  \nabla_\theta \log \pi_\theta (\by_w|\bx) - \nabla_\theta \log \pi_\theta(\by_l|\bx) \right) \\
        &= \omega + \eta \left(\nabla_\theta \log \pi_\theta(\by_w | \bx) \cdot c_1 - \nabla_\theta \log \pi_\theta(\by_l|\bx) \cdot c_2 \right)^\top \left(  \nabla_\theta \log \pi_\theta (\by_w|\bx) - \nabla_\theta \log \pi_\theta(\by_l|\bx) \right) \\
        &= \omega + \eta \left[ c_1 \left\vert \left\vert \nabla_\theta \log \pi_\theta(\by_w|\bx) \right\vert \right\vert^2 + c_2 \left\vert \left\vert \nabla_\theta \log \pi_\theta(\by_l|\bx) \right\vert \right\vert^2 - (c_1 + c_2) \nabla_\theta \log \pi_\theta(\by_l|\bx)^\top \nabla_\theta \log \pi_\theta(\by_w|\bx) \right].
    \end{align*}
    Now, define: $f(t+1; \by_w, \by_l, \bx) = \omega_{t+1} \Big\vert_{{c_2 > 0}} -  \omega_{t+1} \Big\vert_{{c_2 = 0}} $, then we have:
    \begin{align*}
       f(t+1; \by_w, \by_l, \bx) &=  f(t; \by_w, \by_l, \bx) + \eta \underbrace{\left[ c_1 \left\vert \left\vert \nabla_\theta \log \pi_\theta(\by_l|\bx) \right\vert \right\vert^2 - c_2 \nabla_\theta \log \pi_\theta(\by_l|\bx)^\top \nabla_\theta \log \pi_\theta(\by_w|\bx) \right]}_{\Delta(\by_l, \by_w, \bx)}.
    \end{align*}
    Suppose that for all negatives $\by_l$ for a given positive response $\by_w$, $f(t+1; \by_w, \by_l, \bx) < 0$, then:
    \begin{align*}
        \forall \by_l &~~ \Delta(\by_w, \by_l, \bx) < 0 \\
        \implies & \forall \by_l, ~~ c_2 \nabla_\theta \log \pi_\theta(\by_l|\bx)^\top \nabla_\theta \log \pi_\theta(\by_w|\bx) > c_1 \left\vert \left\vert \nabla_\theta \log \pi_\theta(\by_l|\bx) \right\vert \right\vert^2 \\
        \implies &~~ c_2 \mathbb{E}_{\by_l \sim \pi_\theta(\by_l|\bx)} \left[ \nabla_\theta \log \pi_\theta(\by_l|\bx)^\top \nabla_\theta \log \pi_\theta(\by_w|\bx) \right] > c_1 \mathbb{E}_{\by_l \sim \pi_\theta(\by_l|\bx)} \left[ \left\vert \left\vert \nabla_\theta \log \pi_\theta(\by_l|\bx) \right\vert \right\vert^2 \right] \\
        \implies &~~ c_2 \mathbb{E}_{\by_l \sim \pi_\theta(\by_l|\bx)} \left[ \nabla_\theta \log \pi_\theta(\by_l|\bx) \right]^\top \nabla_\theta \log \pi_\theta(\by_w|\bx) > c_1 \mathbb{E}_{\by_l \sim \pi_\theta(\by_l|\bx)} \left[ \left\vert \left\vert \nabla_\theta \log \pi_\theta(\by_l|\bx) \right\vert \right\vert^2 \right] \\
        \implies &~~ 0 > c_1 \mathbb{E}_{\by_l \sim \pi_\theta(\by_l|\bx)} \left[ \left\vert \left\vert \nabla_\theta \log \pi_\theta(\by_l|\bx) \right\vert \right\vert^2 \right],
    \end{align*}
    which is a contradiction since $c_1 > 0$. This means that there is at least one choice of $\by_l$ for a given $\by_w$, for which $\Delta(\by_w, \by_l, \bx) \geq 0$. This means that if $f(t; \by_w, \by_l, \bx) > 0$ then $f(t+1; \by_w, \by_l, \bx) > 0$. Averaging over $\bx$ for all iterations then gives us the desired result, when starting from an initialization when starting from the same initialization for both the cases when $c_2 > 0$ and $c_2 = 0$.

    For the second part of this result, we note that when the gradient dot products are negative and $c_2 > 0$, then by writing down the Taylor expansion, we can note that the likelihood of the positive sample increases by an additional $-c_2 \mathbb{E}_{\bx, \by_w, \by_l \sim \mathcal{D}} \left[\nabla_\theta \log \pi_\theta(\by_l|\bx) ^\top \nabla_\theta \log \pi_\theta(\by_w|\bx)\right]$ and decreases by an additional amount given by $c_2 \mathbb{E}_{\bx, \by_w, \by_l \sim \mathcal{D}} \left[ \left\vert \left\vert \nabla_\theta \log \pi_\theta(\by_l|\bx) \right\vert \right\vert^2 \right]$ for the negative response. This proves the second part of this statement.
\end{proof}

\textbf{Gradients for both DPO and IPO exhibit the form in \cref{lemma:negative_grad}.} We now show that the gradient of both DPO and IPO takes the form shown in~\cref{eq:update_negative_gradient}. From~\citet{rafailov2023direct}, the gradient of the DPO loss is:
\begin{align*}
    \nabla_\theta \mathcal{L}_{\text{DPO}}(\pi_\theta;\piref) = - \beta \mathbb{E}_{(\bx, \by_w, \by_l) \sim \mathcal{D}_\text{pref}} \left[ c^{\text{DPO}}(\bx, \by_w, \by_l) \cdot \left[\nabla_\theta \log \pi_\theta(\by_w | \bx) - \nabla_\theta \log \pi_\theta(\by_l|\bx)\right]\right]
\end{align*}

where $c^{\text{DPO}}(\bx, \by_w, \by_l) = \sigma\left(\beta \log \frac{\pi_\theta(\by_l|\bx)}{\piref\by_l|\bx)} - \beta \log \frac{\pi_\theta(\by_w|\bx)}{\piref\by_w|\bx)} \right)$.

Now we derive the gradient of the IPO loss. Define
\begin{align*}
    c^{\text{IPO}}(\bx, \by_w, \by_l) = 2 \cdot \left(\log\left(\frac{\pi_{\theta}(\y_w|x)}{\piref(\y_w|x)}\right) - \log\left(\frac{\pi_{\theta}(\y_l|x)}{\piref(\y_l|x)}\right) - \frac{\tau^{-1}}{2}\right)
\end{align*}
The gradient of the IPO loss is:
\begin{align*}
    \nabla_\theta \mathcal{L}_{\text{IPO}}(\pi_\theta;\piref) = {} & \nabla_\theta \mathbb{E}_{(\bx, \by_w, \by_l) \sim \mathcal{D}_{\text{pref}}} \left[\left(\log\left(\frac{\pi_{\theta}(\y_w|x)\piref(\y_l|x)}{\piref(\y_w|x)\pi_{\theta}(\y_l|x)}\right) - \frac{\tau^{-1}}{2}\right)^2\right] \\
    = {} & \mathbb{E}_{(\bx, \by_w, \by_l) \sim \mathcal{D}_{\text{pref}}} \left[c^{\text{IPO}}(\bx, \by_w, \by_l) \cdot \left[ \nabla_\theta \log \pi_\theta(\by_w|x) - \nabla_\theta \log \pi_\theta(\by_l|x) \right]\right]
\end{align*}

\subsubsection{Supervised Offline Algorithms are Mode-Covering}\label{appendix:supervised_offline}

Now we prove~\cref{lemma:supervised_offline}, which shows that supervised offline methods that optimize a maximum likelihood loss exhibit mode-covering behavior.

\begin{proof}
    Offline supervised methods optimize the following loss function:
    \begin{align*}
        \mathcal{L}_{\text{off-sup}}(\pi_\theta; \piref) = {} & - \mathbb{E}_{\bx \sim \mathcal{D}_\text{pref}}\left[ \sum_{\by}\piref(\by|\bx) \log \pi_\theta(\by|\bx) \cdot F(\bx, \by)\right]
    \end{align*}

    Define a new distribution $$\tilde{\pi}(\by|\bx) = \frac{\piref(\by|\bx) \cdot F(\bx, \by)}{Z(\bx)}$$ Here $Z(\bx) = \sum_{\bz} \piref(\bz|\bx) \cdot F(\bx, \bz)$ is the normalization constant. It is easy to check that this a valid conditional distribution. This gives us:
    \begin{align*}
        \mathcal{L}_{\text{off-sup}}(\pi_\theta; \piref) = {} & - \mathbb{E}_{\bx \sim \mathcal{D}_\text{pref}}\left[ Z(\bx) \sum_{\by}\tilde{\pi}(\by|\bx) \log \pi_\theta(\by|\bx) \right] \\
        = {} & \mathbb{E}_{\bx \sim \mathcal{D}_\text{pref}}\left[ Z(\bx) \cdot \mathbb{E}_{\by \sim \tilde{\pi}(\by|\bx)}\left[\log \left(\frac{\tilde{\pi}(\by|\bx)}{\pi_\theta(\by|\bx)}\right) \right]\right] - \mathbb{E}_{\bx \sim \mathcal{D}_\text{pref}}\left[ Z(\bx) \cdot \mathbb{E}_{\by \sim \tilde{\pi}(\by|\bx)}\left[\log \tilde{\pi}(\by|\bx) \right]\right] \\
        = {} & \mathbb{E}_{\bx \sim \mathcal{D}_\text{pref}}\left[ Z(\bx) \cdot \mathbb{D}_{\text{KL}}(\tilde{\pi}(.|\bx) || \pi_\theta(.|\bx))\right] + \mathbb{E}_{\bx \sim \mathcal{D}_\text{pref}}\left[Z(\bx) \cdot H(\tilde{\pi}(.|\bx))\right]
    \end{align*}
    Hence offline supervised methods minimize the re-weighted forward KL-divergence.
\end{proof}

\subsection{Characterization of Gradients of Forward and Reverse KL} \label{appendix:forward_reverse_kl_gradient_characterization}

For simplicity, let $\mathcal{X}$ be our input space, and $\mathcal{Y} = \{1, \ldots, V\}$ be the output space consisting of $V > 1$ discrete tokens. Let $f: \mathcal{X} \rightarrow \mathbb{R}^V$ be our network that outputs $V$ real values logits for inputs $x \in \mathcal{X}$. With respect to the logits of $f$, we can define a probability distribution $p$ over discrete tokens $1, \ldots, V$ using the softmax function, namely:

\begin{align}
\label{eq:softmax_prob}
    p_i(x) = \frac{\exp(f_i(x))}{\sum_{k = 1}^V \exp(f_k(x))}
\end{align}

for any $x \in \mathcal{X}$.

Now assume that for some given input $x$, $q(x)$ is the true probability distribution over tokens that we want to learn. One way to do this would be to find $p$ that minimizes $\mathbb{D}_{\text{KL}}(q || p)$ via SGD:
\begin{equation*}
    p_0(x) \leftarrow p_{\text{ref}}(x)
\end{equation*}
\begin{equation*}
    p_{t + 1}(x) \leftarrow p_{t}(x) - \eta \nabla_{f(x)} \mathbb{D}_{\text{KL}}(q(x) || p(x))\big|_{p = p_{t}}
\end{equation*}

where $p_{\text{ref}}$ is reference distribution we start with, and  
$\eta$ is the learning rate. $\mathbb{D}_{\text{KL}}(q || p)$ is called the \textbf{forward-KL}, since $q$ is the true distribution of interest. In contrast, $\mathbb{D}_{\text{KL}}(p || q)$ is called the \textbf{reverse-KL}, and one can try to optimize $p$ by minimizing $\mathbb{D}_{\text{KL}}(p || q)$ in a similar fashion. 

We shall now prove~\cref{lemma:gradient_characterization_forward_vs_reverse_kl}. We break this lemma in two parts. First, let us investigate how the gradients $\nabla_{f(x)} \mathbb{D}_{\text{KL}}(q(x) || p(x))\big|_{p = p_{t}}$ and $\nabla_{f(x)} \mathbb{D}_{\text{KL}}(p(x) || q(x))\big|_{p = p_{t}}$ look like:

\begin{lemma}
The gradients of forward and reverse KL are given by: 
    \begin{align} \label{eq:derivative_forward_kl}
    \frac{\partial}{\partial f_j} \mathbb{D}_{\text{KL}}(q(x) || p(x)) = p_j(x) - q_j(x)
\end{align}
    \begin{align} \label{eq:derivative_reverse_kl}
    \frac{\partial}{\partial f_j} \mathbb{D}_{\text{KL}}(p(x) || q(x)) = p_j(x) \left[ \log \frac{p_j(x)}{q_j(x)} - \mathbb{D}_{\text{KL}}(p(x) || q(x))\right]
\end{align}

\end{lemma}
\begin{proof}
    We start with the definition of KL-divergence:
    \begin{align*}
        \mathbb{D}_{\text{KL}}(q(x) || p(x)) = {} & \sum_{i}q_i(x)\log q_i(x) - \sum_i q_i(x) \log p_i(x) \\
        = {} & -H(q) - \sum_i q_i(x) \log\left(\frac{e^{f_i(x)}}{\sum_k e^{f_k(x)}}\right) \\
        = {} & -H(q) - \sum_i q_i(x)f_i(x) + \sum_i q_i(x) \log \left(\sum_k e^{f_k(x)}\right)
    \end{align*}
    Therefore, we have:
    \begin{align*}
        \frac{\partial}{\partial f_j} \mathbb{D}_{\text{KL}}(q(x) || p(x)) = {} & -q_j(x) + \sum_i q_i(x)\left( \frac{e^{f_j(x)}}{\sum_k e^{f_k(x)}}\right) \\
        = {} & -q_j(x) + \sum_i q_i(x) p_j(x) \\
        = {} & p_j(x) - q_j(x)
    \end{align*}
    This proves \cref{eq:derivative_forward_kl}.
    Similarly, we can write:
    \begin{align*}
        \mathbb{D}_{\text{KL}}(p(x) || q(x)) = {} & \sum_i p_i(x) \log p_i(x) - \sum_i p_i(x) \log q_i(x) \\
        = {} & \sum_i \left(\frac{e^{f_i (x)}}{\sum_k e^{f_k(x)}}\right)\left[f_i(x) - \log\left(\sum_k e^{f_k(x)}\right)\right] - \sum_i \log q_i(x) \left(\frac{e^{f_i (x)}}{\sum_k e^{f_k(x)}}\right) \\
        = {} & \frac{\sum_i f_i(x) e^{f_i(x)}}{\sum_k e^{f_k(x)}} - \left(\sum_i e^{f_i(x)} \right) \left(\frac{\log\left(\sum_k e^{f_k(x)}\right)}{\sum_k e^{f_k(x)}}\right) - \sum_i \log q_i(x) \left(\frac{e^{f_i (x)}}{\sum_k e^{f_k(x)}}\right) \\
        = {} & \frac{\sum_i f_i(x) e^{f_i(x)}}{\sum_k e^{f_k(x)}} - \log\left(\sum_k e^{f_k(x)}\right) -  \frac{\sum_i \log q_i(x)e^{f_i (x)}}{\sum_k e^{f_k(x)}}
    \end{align*}
    Now we calculate the partial derivative with respect to $f_j$:
    \begin{align*}
        \frac{\partial}{\partial f_j} \frac{\sum_i f_i(x) e^{f_i(x)}}{\sum_k e^{f_k(x)}} = {} & \frac{\left(\sum_k e^{f_k(x)}\right) \frac{\partial}{\partial f_j} \left(\sum_i f_i(x) e^{f_i(x)}\right) - \left(\sum_i f_i(x) e^{f_i(x)}\right) \left(\frac{\partial}{\partial f_j} \sum_k e^{f_k(x)}\right)}{\left(\sum_k e^{f_k(x)}\right)^2} \\
        = {} & \frac{\frac{\partial}{\partial f_j} \left(\sum_i f_i(x) e^{f_i(x)}\right)}{\sum_k e^{f_k(x)}} - \frac{e^{f_j(x)}\left(\sum_i f_i(x) e^{f_i(x)}\right)}{\left(\sum_k e^{f_k(x)}\right)^2} \\
        = {} & \frac{e^{f_j(x)} + f_j(x)e^{f_j(x)}}{\sum_k e^{f_k(x)}} - \frac{e^{f_j(x)}}{\sum_k e^{f_k(x)}} \left(\sum_i f_i(x)\left(\frac{e^{f_i(x)}}{\sum_k e^{f_k(x)}}\right)\right) \\
        = {} & p_j(x) + f_j(x)p_j(x) - p_j(x)\left(\sum_i f_i(x)p_i(x)\right)
    \end{align*}
    \begin{equation*}
     \frac{\partial}{\partial f_j} \log\left(\sum_k e^{f_k(x)}\right) = \frac{e^{f_j(x)}}{\sum_k e^{f_k(x)}} = p_j(x)
    \end{equation*}

    And for the third term,
    \begin{align*}
        \frac{\partial}{\partial f_j} \frac{\sum_i \log q_i(x)e^{f_i (x)}}{\sum_k e^{f_k(x)}} = {} &  \frac{\left(\sum_k e^{f_k(x)}\right) \frac{\partial}{\partial f_j} \left(\sum_i \log q_i(x)e^{f_i (x)}\right) - \left(\sum_i \log q_i(x)e^{f_i (x)}\right) \frac{\partial}{\partial f_j} \left(\sum_k e^{f_k(x)}\right)}{\left(\sum_k e^{f_k(x)}\right)^2} \\
        = {} & \frac{e^{f_j(x)}\log q_j(x)}{\sum_k e^{f_k(x)}} - \left(\frac{e^{f_j(x)}}{\sum_k e^{f_k(x)}}\right) \left(\sum_i \log q_i(x) \frac{e^{f_i(x)}}{\sum_k e^{f_k(x)}}\right) \\
        = {} & p_j(x) \log q_j(x) - p_j(x) \sum_i p_i(x) \log q_i(x)
    \end{align*}

Putting it all together, we obtain:
\begin{align*}
    \nabla_{f_j} \mathbb{D}_{\text{KL}}(p(x) || q(x)) = {} & p_j(x) \cdot \left(f_j(x) - \log q_j(x) \right) - p_j(x) \cdot \left( \sum_{i} p_i(x) \cdot \left(f_i(x) - \log q_i(x) \right) \right) \\
    = {} & p_j(x) \cdot \log \frac{p_j(x)}{q_j(x)} - p_j(x) \sum_{i} p_i(x) \cdot \log \frac{p_i(x)}{q_i(x)} \\
    = {} & p_j(x) \left[ \log \frac{p_j(x)}{q_j(x)} - \mathbb{D}_{\text{KL}}(p(x) || q(x))\right]
\end{align*}
completing our proof.
\end{proof}

\textbf{Proof for Lemma~\ref{lemma:gradient_characterization_forward_vs_reverse_kl}}. Now, if the logits $f_t$ are being updated with gradient descent on loss $\mathcal{L}$, the distribution at the next step $p^{t+1}$ is given by:
\begin{align*}
    p^{t+1}_j(x) = {} & \exp (f^{t+1}_j(x)) / \sum_{i} \exp (f^{t+1}_i(x)) \\ 
    = {} & \frac{\exp (f^t_j(x)) - \eta \nabla_{f^t_j} \mathcal{L})}{\sum_i \exp (f^t_i(x))} \cdot \frac{\sum_i \exp (f^t_i(x))}{ \sum_{i} \exp (f^t_i(x) - \eta \nabla_{f^t_i} \mathcal{L})}\\
    = {} & p^t_j(x) \cdot \frac{\exp \left( -\eta \nabla_{f^t_j} \mathcal{L} \right)}{\sum_i p^t_i(x) \exp \left( -\eta \nabla_{f^t_i} \mathcal{L} \right) }
\end{align*}

Let's consider what the characterization of $p^{t+1}$ for the forward kl:
\begin{align*}
    p^{t+1}_j(x) &= p^t_j(x) \cdot \frac{\exp \left( -\eta \left(p_j^t(x) - q_j(x)\right) \right)}{\sum_i p^t_i(x) \exp \left( -\eta \left(p_i^t(x) - q_i(x)\right) \right) }
\end{align*}

Noticing that the denominator is just a normalization constant, we can write this as:
    \begin{align} \label{eq:update_ratio_forward_kl}
     \frac{p^{t+1}_j(x)}{p^t_j(x)} \propto  \exp \left( -\eta \left(p_j^t(x) - q_j(x)\right) \right)
\end{align}

Similarly the characterization of $p^{t+1}$ for the reverse KL looks like:
    \begin{align} \label{eq:update_ratio_reverse_kl}
    \frac{p^{t+1}_j(x)}{p^t_j(x)} &\propto  \exp \left( -\eta \left(p^t_j(x)\left[\log \frac{p_j^t(x)}{q^t(x)} - \mathbb{D}_{\text{KL}}(p^t(x) || q(x))\right]\right) \right)
    \end{align}

This completes the proof of~\cref{lemma:gradient_characterization_forward_vs_reverse_kl}.

\subsection{Quantifying the Differences Between Forward and Reverse KL} \label{appendix:quantifying_dufference}

In this section we will prove~\cref{thm:commital}: specifically, we will study certain special cases to explain the differences between approaches that optimize the forward and reverse KL divergences. We drop the subscript $t$ from all terms to prevent notational clutter.

\begin{proof}
    We prove these statements case by case. First we prove the result for Case 1. In this scenario, we have the following:
    \begin{align*}
        \Delta^f(\bx_1, \bx_2) &= \eta \left(q(\bx_1) - q(\bx_2) \right) \\
        \Delta^r(\bx_1, \bx_2) &= \eta p(\bx_1) \left[ \log q(\bx_1) - \log q(\bx_2) \right].
    \end{align*}
    The gap between $\Delta^f$ and $\Delta^r$ is now given by:
    \begin{align*}
        \Delta^r(\bx_1, \bx_2) - \Delta^f(\bx_1, \bx_2) &= \eta \left[ \log q(\bx_1) - \log q(\bx_2) - \frac{q(\bx_1) - q(\bx_2)}{p(\bx_1)} \right].
    \end{align*}
    Now, we note by mean-value theorem, that there exists a $c_0 \in [q(\bx_2), q(\bx_1)]$ such that, 
    \begin{align*}
        \log q(\bx_1) - \log q(\bx_2) = \frac{d \log p}{dp} \Big\vert_{p = c_0} \cdot \left( q(\bx_1) - q(\bx_2) \right).
    \end{align*}
    Since $d \log p / d p = 1/ p > 1$ for $c_0 \in (0, 1)$, we have that:
    \begin{align*}
        \Delta^r(\bx_1, \bx_2) - \Delta^f(\bx_1, \bx_2) = \eta \cdot \left( q(\bx_1) - q(\bx_2) \right) \cdot \left[ \frac{1}{c_0} - \frac{1}{p(\bx_1)} \right].
    \end{align*}
    This quantity is positive when $p(\bx_1) > c_0 = \delta_1$. This shows the result for Case 1.

    \textbf{Next we prove Case 2.} In this setting we are given $q(\bx_1) = q(\bx_2) \geq p(\bx_1) \geq p(\bx_2) + \beta$. In this case, the expressions for $\Delta^f$ and $\Delta^r$ are given by:
    \begin{align*}
        \Delta^f(\bx_1, \bx_2) = -\eta \left( p(\bx_1) - p(\bx_2) \right) \leq - \eta \beta.
    \end{align*}
    On the other hand, the expression for $\Delta^r(\bx_1, \bx_2)$ is given by:
    \begin{align}
 \Delta^r(\bx_1, \bx_2) &= \eta \underbrace{\left[ p(\bx_1) - p(\bx_2) \right] \log q(\bx_1)}_{(a)} - \eta \underbrace{\left[ p(\bx_1) \log p(\bx_1) - p(\bx_2) \log p(\bx_2) \right]}_{(b)} \nonumber \\
 &~~~~~~~~~~~~~~~~~~~~~~~~~~~~+\eta \mathrm{D}_{\text{KL}}(p, q) \underbrace{\left(p(\bx_1) - p(\bx_2) \right)}_{\geq 0}. \label{eq:tmp2}
    \end{align}
    Now we analyze each sub-term independently. First, we note the following expression for term (b):
    \begin{align*}
        (b) \coloneqq &~ p(\bx_1) \log p(\bx_1) - p(\bx_2) \log p(\bx_2) \\
        = &~ p(\bx_1) \log p(\bx_1) - p(\bx_2) \log p(\bx_1) + p(\bx_2) \log p(\bx_1) - p(\bx_2) \log p(\bx_2) \\
        = &~ \left( p(\bx_1) - p(\bx_2) \right) \cdot \log p(\bx_1) + p(\bx_2) \cdot \left( \log p(\bx_1) - \log p(\bx_2) \right).
    \end{align*}
    Combining $(a)$ and $(b)$, we get:
    \begin{align}
        (a) + (b) &= \eta \left[ p(\bx_1) - p(\bx_2) \right] \cdot \left[ \log q(\bx_1) - \log p(\bx_1) \right] - \eta p(\bx_2) \left[ \log p(\bx_1) - \log p(\bx_2) \right)] \nonumber \\
        &= \eta \left( \left[p(\bx_1) - p(\bx_2) \right] \cdot \left[ \log q(\bx_1) - \log p(\bx_1) - p(\bx_2) \cdot \frac{1}{c'} \right] \right), \label{eq:tmp1}
    \end{align}
    where $c'$ is obtained by applying the mean value theorem on the difference $\log p(\bx_1) - \log p(\bx_2)$. Now, since $q(\bx_1) \geq \mathbf{c}_0 \cdot p(\bx_1)$, $\log q(\bx_1) - \log p(\bx_1) \geq \log \mathbf{c}_0$. Hence, if $p(\bx_2)$ is upper bounded (i.e., when $\beta$ is large enough), then this difference $(a) + (b)$ in Equation~\ref{eq:tmp1} is positive. Combining with Equation~\ref{eq:tmp2}, we note that: $\Delta^r(\bx_1, \bx_2) > 0$, although $\Delta^f(\bx_1, \bx_2) < 0$. This concludes the proof. 

    \textbf{Next, we prove Case 3.} Similar to the previous case, here $\Delta^f(\bx_1, \bx_2) = - \eta (p(\bx_1) - p(\bx_2)) \leq -\eta \beta < 0$. In this case, expanding upon the expression of $\Delta^r(\bx_1, \bx_2)$ similarly as Case 2, in order to show the desired inequality $\Delta^r(\bx_1, \bx_2) < \Delta^f(\bx_1, \bx_2)$, we need to prove that:
    \begin{align*}
        \left( p(\bx_1) - p(\bx_2) \right) \cdot \log q(\bx_1) ~&\leq p(\bx_1) \log p(\bx_1) - p(\bx_2) \log p(\bx_2) + \alpha_0,
    \end{align*}
    where $\alpha_0$ subsumes the terms $- \beta$ and $\mathrm{D}_{\text{KL}}(p, q) \cdot (p(\bx_1) - p(\bx_2))$. By applying mean value theorem, on the RHS of this equation, we note that:
    \begin{align*}
        p(\bx_1) \log p(\bx_1) - p(\bx_2) \log p(\bx_2) = \left(1 + \log c'' \right) \cdot \left(p(\bx_1) - p(\bx_2)\right), ~~~ c'' \in [p(\bx_2), p(\bx_1)].
    \end{align*}
    Then, to attain the desired inequality, we need:
    \begin{align*}
        \left[ p(\bx_1) - p(\bx_2) \right] \cdot \left[ \log q(\bx_1) - 1 - \log c'' \right] \leq \alpha_0.
    \end{align*}
    Note that since $c'' \geq p(\bx_2)$, as long as there exists a sufficiently small constant $\mathbf{c}_1 < 1$, such that:
    \begin{align*}
        q(\bx_1) &\leq \mathbf{c}_1 \cdot p(\bx_2) \leq c_1 \cdot c''\\
        &\implies \log q(\bx_1) \leq \log \mathbf{c}_1 + \log c'',
    \end{align*}
    the LHS of this equation will be smaller than the RHS $\alpha_0$. This proves the result for this case.
\end{proof}

\vspace{-0.2cm}

\section{Additional Algorithmic Details}

\subsection{Score/Reward Standardization} \label{subsection:reward_normalization}
Online methods such as PPO or RWR that uses a learned reward model can suffer from gradient variance issues due to the differences in the reward score. In particular, adding or subtracting a baseline $b$ from the reward $r_\phi(\bx, \by)$ does not change the relative order of preferred or dispreferred responses; however, it can change the variance of the gradients, leading to instability of the optimization routine. To mitigate this, prior work~\citep{ziegler2020finetuning} often normalizes the reward to have zero mean and unit variance. This can be done during the training process by computing the mean and variance of the reward from an online batch. Formally, let $\{\bx^{(i)}, \by^{(i)}\}_{i = 1}^\mathcal{B}$ be a batch of data with batch size $\mathcal{B}$ sampled from policy $\pi_{\theta}$: one calculates the standardized reward $\bar{r}_\phi(\bx^{(i)}, \by^{(i)})$ as:
\begin{equation}
    \bar{r}_\phi(\bx^{(i)}, \by^{(i)}) = \frac{r_\phi(\bx^{(i)}, \by^{(i)}) - \hat{\mu}}{\hat{\sigma}}
\end{equation}
where $\hat{\mu} = \frac{1}{\mathcal{B}}\sum_{i = 1}^\mathcal{B} r_\phi(\bx^{(i)}, \by^{(i)})$, $\hat{\sigma} = \sqrt{\frac{1}{\mathcal{B} - 1} \sum_{i = 1}^\mathcal{B} (r_\phi(\bx^{(i)}, \by^{(i)})^2 - \hat{\mu})^2}$.

\subsection{IPO}\label{subsection:ipo}

IPO~\citep{2023arXiv231012036G} is a contrastive algorithm similar to DPO. The key difference between them is their loss function: DPO optimizes the negative log-sigmoid loss whereas IPO optimizes an MSE-type objective. Formally, the IPO objective is:
\begin{equation} \label{eq:ipo_objective}
    \mathcal{L}_{\text{IPO}}(\pi_{\theta};\piref) = \mathbb{E}_{(\bx, \by_w, \by_l) \sim \mathcal{D}_{\text{pref}}} \left(\log\left(\frac{\pi_{\theta}(\y_w|x)\piref(\y_l|x)}{\piref(\y_w|x)\pi_{\theta}(\y_l|x)}\right) - \frac{\tau^{-1}}{2}\right)^2
\end{equation}

where $\tau$ is a hyperparameter controlling how much the learned policy $\pi_\theta$ deviates from the reference policy $\piref$.

\section{Method Hyperparameters} \label{subsection:hyperparam}
We did an extensive sweep over hyperparameters for individual offline and online algorithms for the language model experiments. We built our algorithm implementations off of the Huggingface TRL implementation~\citep{vonwerra2022trl}.

\subsection{Standardized Parameters (Consistent for all Methods)}
\begin{table}[H]
\centering
\caption{Algorithm Agnostic Hyperparamters}
\label{tab:algorithm_agnostic_hyperparams}
\catcode`,=\active
\def,{\char`,\allowbreak}
\begin{tabular}{p{3cm} p{2.5cm} p{8cm}}
\toprule
\textbf{Hyperparameters} & \textbf{Values} & \textbf{Description} \\
\midrule
$B$ & 64 & Batch Size \\
$B_{mini}$ & 8 & Mini-Batch Size \\
$G$ & 8 & Gradient Accumulation Steps \\
$\hat{\pi}_{\theta}$ & Pythia1.4B, Mistral-7b & Policy Architecture \\
$\hat{R}_{\theta}$ & Pythia410M, Mistral-7B & Reward Model Architecture \\
optimizer & Adam & Gradient Optimizer \\
\bottomrule
\end{tabular}
\end{table}

\begin{table}[H]
\centering
\caption{Sampling Hyperparamters}
\label{tab:generation_kwargs}
\catcode`,=\active
\def,{\char`,\allowbreak}
\begin{tabular}{p{3cm} p{2.5cm} p{8cm}}
\toprule
\textbf{Hyperparameters} & \textbf{Values} & \textbf{Description} \\
\midrule
top\_k           & 0.0   & Disables top-k sampling \\
top\_p           & 1.0   & Disables nucleus sampling \\
do\_sample       & True  & Enables sampling \\
max\_new\_tokens & 256   & Maximum number of new tokens to generate \\
temperature      & 1.0   & Sets sampling temperature (1.0 for default) \\
use\_cache       & True  & Uses past key/values attentions if supported by the model \\
\bottomrule
\end{tabular}
\end{table}

\subsection{DPO \citep{rafailov2023direct}}
\begin{table}[H]
\centering
\caption{DPO Hyperparameters}
\label{table:hyper_dpo} 
\catcode`,=\active
\def,{\char`,\allowbreak}
\renewcommand\arraystretch{1.2}
\begin{tabular}{p{3cm} p{3.5cm} p{7cm}}
  \toprule
    \textbf{Hyperparameters} & \textbf{Values} & \textbf{Description} \\   
  \midrule
    lr & 1e-7, 5e-7, 1e-6, 5e-6, 1e-5 & learning rate \\
    $\beta$ & $0.01$, $0.05$, $0.1$, $0.5$ & KL weight \\
  \bottomrule
\end{tabular} 
\end{table}

\subsection{Pref-FT \citep{dubois2024alpacafarm}}
\begin{table}[H]
\centering
\caption{Pref-FT/Binary FeedMe Hyperparameters}
\label{table:hyper_feedme} 
\catcode`,=\active
\def,{\char`,\allowbreak}
\renewcommand\arraystretch{1.2}
\begin{tabular}{p{3cm} p{3.5cm} p{7cm}}
  \toprule
    \textbf{Hyperparameters} & \textbf{Values} & \textbf{Description} \\  
  \midrule
    $\eta$ & 1e-7, 5e-7, 1e-6, 5e-6 & learning rate \\
  \bottomrule
\end{tabular} 
\end{table}

\subsection{PPO \citep{2017arXiv170706347S}}

\begin{table}[H]
\centering
\caption{PPO Hyperparameters}
\label{table:hyper_ppo} 
\catcode`,=\active
\def,{\char`,\allowbreak}
\renewcommand\arraystretch{1.2}
\begin{tabular}{p{3cm} p{3.5cm} p{7cm}}
  \toprule
    \textbf{Hyperparameters} & \textbf{Values} & \textbf{Description} \\    
  \midrule
    $\eta$ & 1e-7, 5e-7, 1e-6, 5e-6, 1e-5 & Learning rate. \\
    vf\_coef & 0.1 & Coefficient for the value function loss. \\
    adap\_kl\_ctrl & True & Enables adaptive KL penalty control. \\
    init\_kl\_coef & 0.2 & Initial coefficient for KL penalty. \\
    target\_kl & 0.1 & Target KL divergence for policy updates. \\
    $N$ & 1 & actions per prompt \\
  \bottomrule
\end{tabular} 
\end{table}

\subsection{RWR}

\begin{table}[H]
\centering
\caption{RWR Hyperparameters}
\label{table:hyper_rwr} 
\catcode`,=\active
\def,{\char`,\allowbreak}
\renewcommand\arraystretch{1.2}
\begin{tabular}{p{3cm} p{3.5cm} p{7cm}}
  \toprule
    \textbf{Hyperparameters} & \textbf{Values} & \textbf{Description} \\
  \midrule
    $\eta$ & 1e-7, 5e-7, 1e-6, 5e-6, 1e-5 & learning rate \\
    $\beta$ & $0.1$, $1$, $10$, $20$ & temperature \\
    $N$ & 1 & actions per prompt \\
  \bottomrule
\end{tabular} 
\end{table}

\subsection{Iterated Best-of-N \citep{superhf}}

\begin{table}[H]
\centering
\caption{Iterated BofN Hyperparameters}
\label{table:hyper_iterated_bofn} 
\catcode`,=\active
\def,{\char`,\allowbreak}
\renewcommand\arraystretch{1.2}
\begin{tabular}{p{3cm} p{3.5cm} p{7cm}}
  \toprule
    \textbf{Hyperparameters} & \textbf{Values} & \textbf{Description} \\ 
  \midrule
    $\eta$ & 1e-7, 5e-7, 1e-6, 5e-6, 1e-5 & learning rate \\
    $N$ & $4$, $10$ & actions per prompt \\
  \bottomrule
\end{tabular} 
\end{table}

\section{Code For Running Experiments} \label{section:reproducibility}

We have made the code for this project public in this \href{https://github.com/Asap7772/understanding-rlhf}{repository}. The additional datasets used in our experiments are listed below:

\begin{itemize}
    \item \href{https://huggingface.co/datasets/Asap7772/relabeled_alpacafarm_pythiasft_20K_preference_data_minlength}{Min Length}
    \item \href{https://huggingface.co/datasets/Asap7772/relabeled_alpacafarm_pythiasft_20K_preference_data_modelength}{Mode Length} 
    \item \href{https://huggingface.co/datasets/Asap7772/alpaca_skewexp_minlength_merged}{Skew Length}
    \item \href{https://huggingface.co/datasets/Asap7772/relabeled_alpacafarm_pythiasft_20K_preference_data}{Relabelled AlpacaFarm}
\end{itemize}

We gratefully acknowledge the following codebases: \href{https://github.com/huggingface/trl}{TRL}~\citep{vonwerra2022trl}, \href{https://github.com/ContextualAI/HALOs}{HALOs}~\citep{ethayarajh2023halos}, 
\href{https://github.com/karpathy/minGPT}{minGPT}~\citep{min_gpt},
\href{https://github.com/facebookresearch/drqv2}{DrQ-v2}~\citep{yarats2021drqv2,yarats2021image} and \href{https://github.com/tajwarfahim/proactive_interventions}{PAINT}~\cite{xie2022paint}.

\section{More on Didactic Bandit Problems} \label{section:banditSetupAppendix}

\subsection{Problem Setup}

Here we present details of our didactic bandit problem. The reference policy shown in \cref{fig:bandit_problem_setup} is obtained by collecting 10000 samples from a Cauchy distribution with location $x_0 = -0.7$, scale $\gamma = 0.4$. Next, we clip this samples between the interval $(-1, 1)$, and divide the interval into 100 equally spaced bins. Starting from $-1$, we label these bins $0, \ldots, 99$ sequentially, and calculate the frequency of samples that fell into each bin. Finally, we define,
\begin{equation*}
    \piref(a_i) = \frac{\text{Freq}(bin_i)}{10000}
\end{equation*}

The reward functions $\mathbf{R}_1$ and $\mathbf{R}_2$ are defined as:
\begin{equation*}
    \mathbf{R}_1(a) = \exp\left(-\left(\frac{a - 70}{10}\right)^2\right)
\end{equation*}
and
\begin{equation*}
    \mathbf{R}_2(a) = \exp\left(-\left(\frac{a - 20}{10}\right)^2\right)
\end{equation*}

\subsection{Algorithmic Details} \label{section:banditAlgorithmsAppendix}

In the bandit setting, we consider five algorithms: (1) Best-of-N, (2) IPO, (3) REINFORCE, (4) PPO and (5) RWR.

\subsubsection{Best-of-N}
Best-of-N is similar to SuperHF~\citep{superhf}/ReST~\citep{gulcehre2023reinforced} and in some way their simplification for the bandit setting. Best-of-N collects $N$ actions/responses for a prompt/state $\bx$, namely $\by_1, \by_2, \ldots, \by_N$. Next, we collect the rewards $\{\mathbf{R}(\bx, \by_i)\}_{i = 1}^N$, and based on these rewards, choose the best action $\by_{\text{best}} = \argmax_{\by_i}  \mathbf{R}(\bx, \by_i)$. Finally, the loss function is the negative log-likelihood of this best action.

\begin{equation*}
    \mathcal{L}_{\text{bofn}}(\pi_{\theta}; \bx, \by_1, \ldots, \by_N) = -\log \pi_{\theta}(\by_{\text{best}} | \bx)
\end{equation*}

In both the online and offline setting, we have a fixed set of prompts $\mathcal{D}_{\text{prompts}}$, and we also always start with $\pi_{\theta}$ initialized to $\piref$. Formally, given a policy $\pi$, we can form a training set as:
\begin{equation*}
    \mathcal{D}_{\text{train}}(\mathcal{D}_{\text{prompts}}, \pi) = \{(\bx, \by): \bx \in \mathcal{D}_{\text{prompts}}, \by  = \argmax_{\by_i}  \mathbf{R}(\bx, \by_i) \text{ where } \by_1, \by_2, \ldots, \by_N \sim \piref(.|\bx)\}
\end{equation*}
In the offline setting, we collect a fixed training dataset where actions are sampled from $\piref$, namely $\mathcal{D}_{\text{train}}(\mathcal{D}_{\text{prompts}}, \piref)$. In the online setting, we collect a new training dataset by sampling actions from the current policy $\pi_{\theta}$, namely $\mathcal{D}_{\text{train}}(\mathcal{D}_{\text{prompts}}, \pi_\theta)$, after every $T$ gradient steps, and discard the previous dataset.

To show the efficacy of negative gradient, we can also directly add a term to this loss function minimizing log probability on dispreferred actions. Explicitly, we consider the following loss function:

\begin{align*}
    \mathcal{L}_{\text{bofn + neg-grad}}(\pi_{\theta}; \bx, \by_1, \ldots, \by_N) = -\log \pi_{\theta}(\by_{\text{best}} | \bx) + \beta \sum_{\by_j \neq \by_{\text{best}}} \log \pi_{\theta}(\by_j | \bx)
\end{align*}

where $\beta$ is a hyperparameter that we usually set to $1.0$. We note that in practice this loss can quickly become unstable and proceed to $-\infty$, in practice we only minimize the probability of dispreferred actions if it is above a certain threshold.

\subsubsection{IPO}

In contrast, IPO uses the loss function defined in \cref{eq:ipo_objective}. While regular IPO is an offline algorithm that uses a fixed preference dataset $\mathcal{D}_{\text{pref}}$, since we have access to the true reward function in the bandit setup, we create an online version of this algorithm as well. Here we also have a fixed set of prompts $\mathcal{D}_{\text{prompts}}$, and given a policy $\pi$, we can generate a preference dataset as follows: for each prompt $\bx \in \mathcal{D}_{\text{prompts}}$, we can generate completions $\by_1, \by_2, \ldots, \by_N \sim \pi(.|\bx)$. For any $i \neq j$, without loss of generality, assume $\mathbf{R}(\bx, \by_i) > \mathbf{R}(\bx, \by_j)$. Then $\by_i$ and $\by_j$ are the preferred and dispreferred completions respectively, and we can form a preference dataset with all such $(\bx, \by_w, \by_l)$ tuples.

In the offline setting, the preference dataset is collected by generating samples from the reference policy $\piref$, and kept fixed during training. In the online setting, we generate the preference dataset from the current policy $\pi_\theta$, after every $T$ gradient steps, and discard the previous dataset.

\subsubsection{REINFORCE}

For REINFORCE, we sample $\by \sim \pi_{\theta}(.|\bx)$, calculate the normalized reward $\overline{\mathbf{R}(\bx, \by)}$, and use the following loss:
\begin{align*}
    \mathcal{L}_{\text{REINFORCE}}(\pi_{\theta}; \mathcal{D}_\text{prompts}) = -\mathbb{E}_{\bx \in \mathcal{D}_\text{prompts}} [\mathbb{E}_{\by \sim \pi_{\theta}} [\log \pi_{\theta}(\by|\bx) \overline{\mathbf{R}(\bx, \by)}]]
\end{align*}

\subsubsection{PPO}

For PPO, let $\pi_{\text{gen}}$ be the policy used to generate the responses, and define $r(\bx, \by) = \frac{\pi_{\theta}(\by|\bx)}{\pi_{\text{gen}}(\by|\bx)}$. Then we use the following loss function:
\begin{align*}
    \mathcal{L}_{\text{PPO}}(\pi_{\theta}; \mathcal{D}_\text{prompts}) = -\mathbb{E}_{\bx \in \mathcal{D}_\text{prompts}} \left[\mathbb{E}_{\by \sim \pi_{\theta}} \left[\max\left( r(\bx, \by)\overline{\mathbf{R}(\bx, \by)}, \text{Clip}\left(r(\bx, \by), 1 - \epsilon, 1 + \epsilon\right) \overline{\mathbf{R}(\bx, \by)}\right)\right]\right]
\end{align*}

where $\epsilon > 0$ is a hyperparameter that controls how much we clip off-policy updates.

\subsubsection{RWR}

For RWR, we use the following loss function:
\begin{align*}
    \mathcal{L}_{\text{REINFORCE}}(\pi_{\theta}; \mathcal{D}_\text{prompts}) = -\mathbb{E}_{\bx \in \mathcal{D}_\text{prompts}} \left[\mathbb{E}_{\by \sim \pi_{\theta}} \left[\log \pi_{\theta}(\by|\bx) \exp\left(\frac{\overline{\mathbf{R}(\bx, \by)}}{\beta}\right)\right]\right]
\end{align*}

where $\beta$ is a hyperparameter, usually $\beta = 0.1$ in our experiments unless otherwise noted.

\subsection{Experiment Details}

For all experiments, we use $N = 10$. For negative gradient experiments, we are in the fully offline setting and vary the size of the prompt dataset $\mathcal{D}_{\text{prompts}}$, with $T = 100$ number of gradient steps performed. For on policy sampling experiments, we hold $|\mathcal{D}_{\text{prompts}}| = 10$ randomly sampled prompts from tokens $\{0, \ldots, 99\}$, and vary $T$. We also a new training dataset from the current policy after each $T$ gradient steps, and perform this data collection step $100$ times for all experiments. We set $\tau = 0.05$ for IPO, and search for the optimal learning rate from 0.3, 0.1, 0.03, 0.01, 0.003, 0.001, 0.0003, 0.0001, 0.00003, and 0.00001 for each experiment and use an Adam~\citep{kingma2017adam} optimizer for all experiments. We run each experiment for 5 seeds, and the shaded region in the plots refers to the standard error of the mean obtained from these runs. Finally, to initialize $\pi_\theta$ to $\piref$, we minimize the KL divergence between $\pi_\theta$ and $\piref$ with an Adam optimizer with a learning rate of $0.01$.

For all experiments, we use a small GPT~\citep{radford2018improving,brown2020language}-like transformer architecture (named `GPT-Nano') with 0.9M parameters. We took the implementation from this public repository: \href{https://github.com/karpathy/minGPT}{minGPT}~\citep{min_gpt}.

\section{Additional Experiments on Synthetic LLM Setup}

\subsection{Performance of Various Algorithms on the Mode Length Setting}

\cref{fig:synthetic_llm_mode_length} shows the performance of various algorithms in the mode length setup. We see that all algorithms perform similarly here.

\begin{figure}[h!]
\vspace{-0.2cm}
    \centering
    \includegraphics[width=0.5\columnwidth]{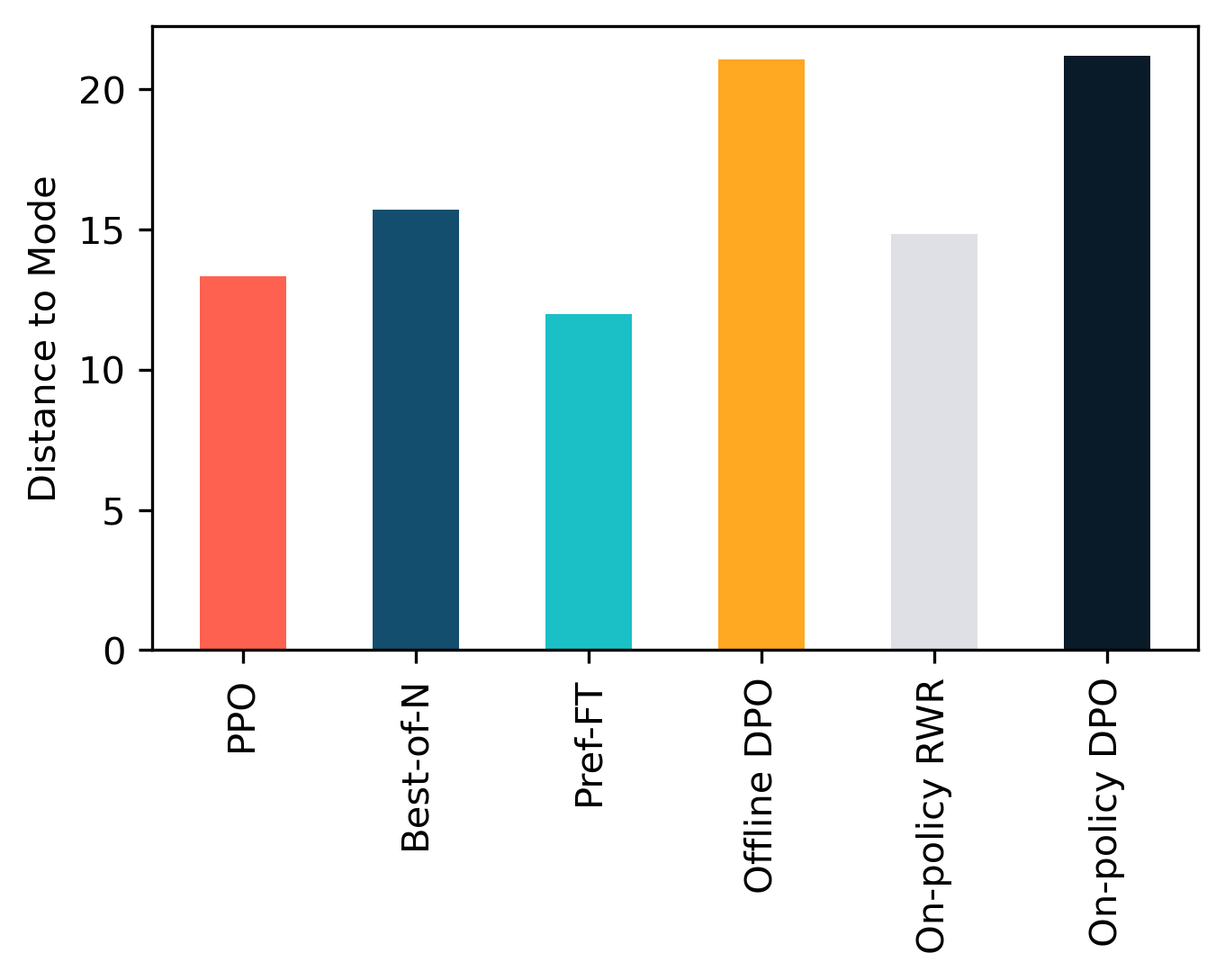}
    \vspace{-0.5cm}
    \caption{\label{fig:synthetic_llm_mode_length}\footnotesize{\textbf{Performance of various algorithms on mode length setup.} Distance to mode of the completion lengths from $\pi_\text{ref}$, 203, for different algorithms. All algorithms perform similarly, and varying degrees of on-policyness does not generally degrade performance.}}
    \vspace{-0.2cm}
\end{figure}

\subsection{Effect of On-policy Samples vs Samples from an Older Policy in Synthetic Length Settings}

\cref{fig:rwr_min_length_batch_size_on_policy,fig:rwr_skew_length_batch_size_on_policy} shows the effect of using on-policy samples vs samples from an older policy for RWR in the synthetic length experiments.

\begin{figure}[h!]
\vspace{-0.2cm}
    \centering
    \includegraphics[width=\columnwidth]{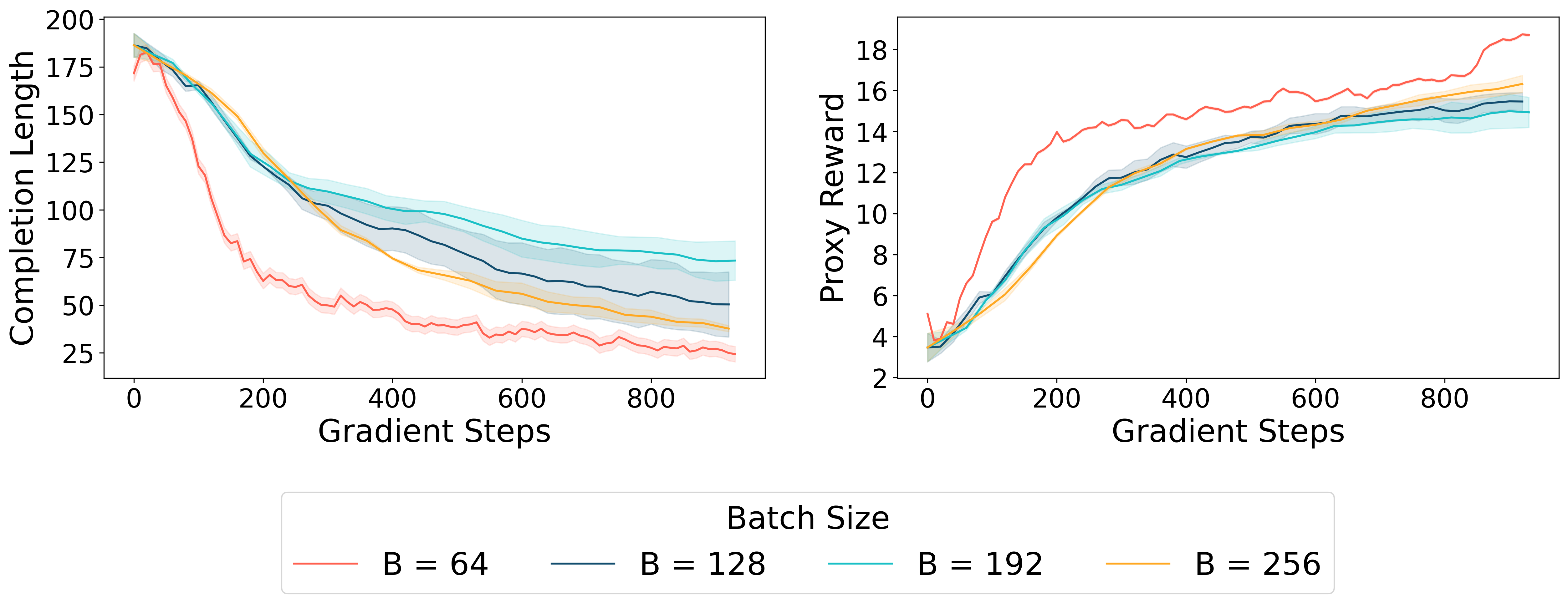}
    \vspace{-0.5cm}
    \caption{\label{fig:rwr_min_length_batch_size_on_policy}\footnotesize{\textbf{On-policy sampling on Min Length (RWR).} Effect of using on-policy samples vs samples from an older policy for RWR and the min length setup. In all experiments, the mini-batch size to calculate the gradient is fixed at 64, and we sample batch size $B$ completions from the current policy, divide it into mini-batches, and take one pass over the entire set of completions before collecting more samples. Increasing $B$ thus makes the algorithm make updates on samples from an older policy. \textbf{Left}: average completion length (lower the better), and \textbf{Right}: proxy reward vs gradient steps. Being more on-policy results in better performance.}}
    \vspace{-0.2cm}
\end{figure}

\begin{figure}[h!]
\vspace{-0.2cm}
    \centering
    \includegraphics[width=\columnwidth]{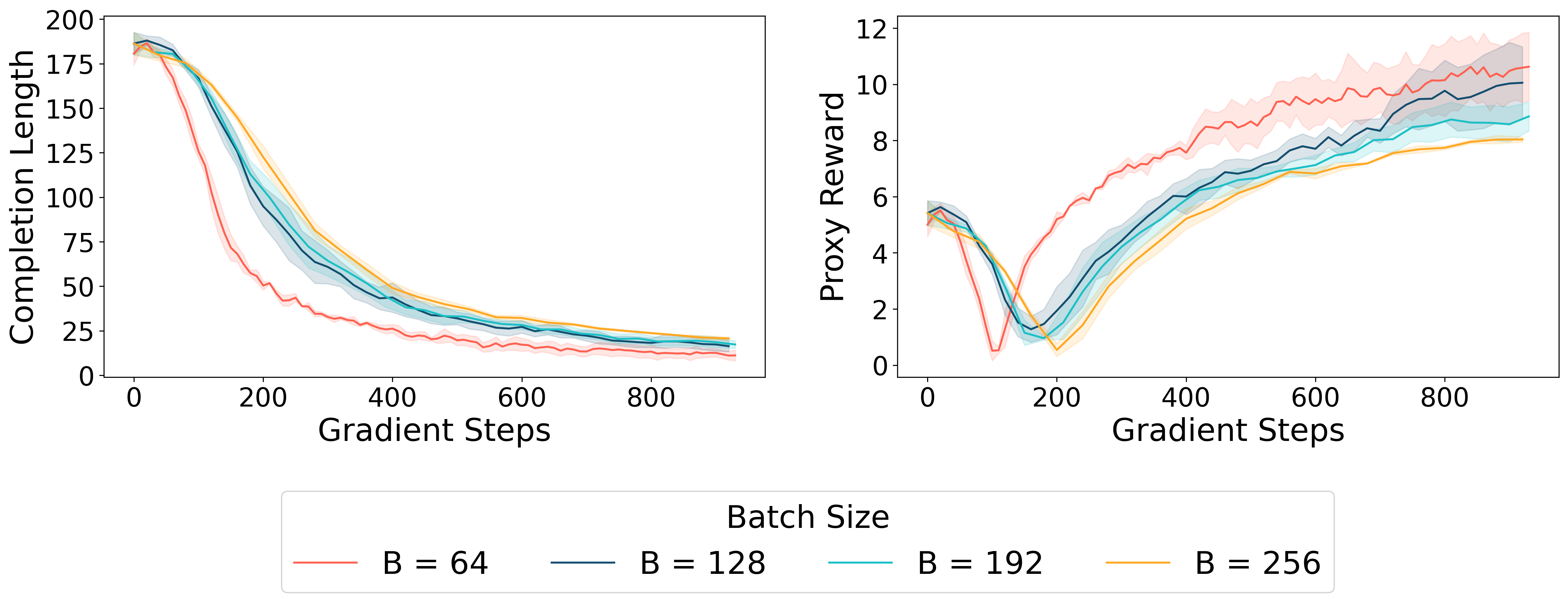}
    \vspace{-0.5cm}
    \caption{\label{fig:rwr_skew_length_batch_size_on_policy}\footnotesize{\textbf{On-policy sampling on Skew Length (RWR).} Effect of using on-policy samples vs samples from an older policy for RWR and the skew length setup. \textbf{Left}: average completion length (lower the better), and \textbf{Right}: proxy reward vs gradient steps. Being more on-policy results in better performance.}}
    \vspace{-0.2cm}
\end{figure}

\subsection{Sample Reuse in Synthetic LLM Settings}

\cref{fig:llm_length_t_ablation_appendix} shows the effect of sample reuse in the \textbf{Skew Length} setting: similar to \textbf{Min Length} (~\cref{fig:llm_length_t_ablation}), some sample reuse can improve sample efficiency. but excessive sample reuse can also hurt performance. Also, we see PPO with importance clipping is much better at sample reuse than Best-of-N.

\begin{figure}[h!]
\vspace{-0.2cm}
    \centering
    \includegraphics[width=0.99\columnwidth]{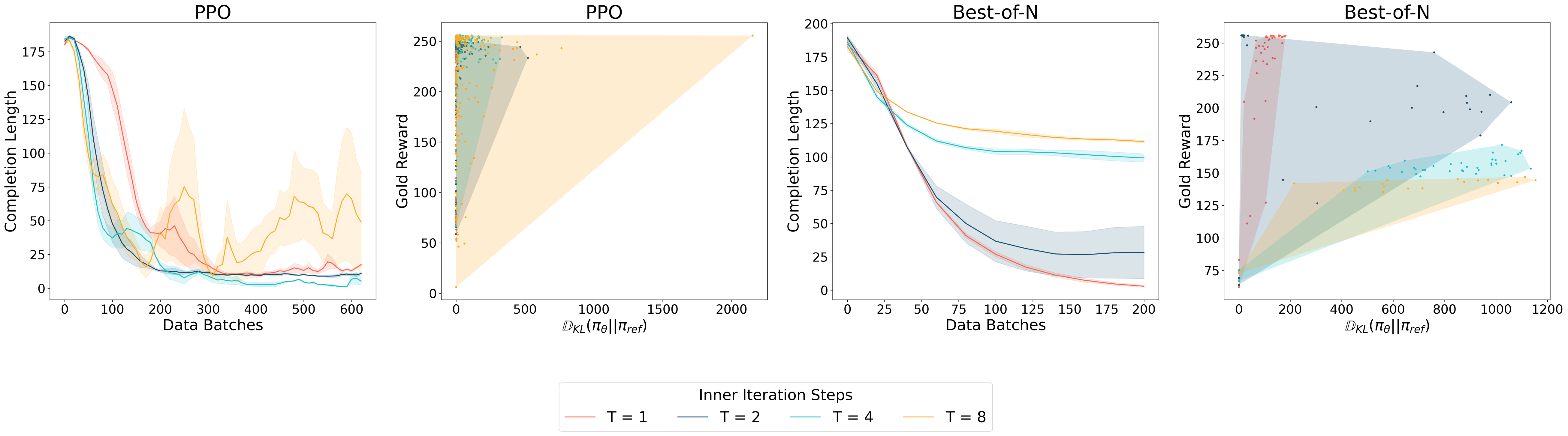}
    \vspace{-0.4cm}
\caption{\label{fig:llm_length_t_ablation_appendix}\footnotesize{\textbf{Effect of on-policy sample reuse in the Skew Length scenario.} Average completion length (i.e., the lower the better) vs gradient steps for different numbers of inner iteration steps, $T$, on the same data batch. A larger value of $T$ implies that the algorithm is more off-policy. Observe that some sample reuse can improve sample efficiency (T = 2 and T = 4 outperform T = 1), but excessive sample reuse can hurt performance (T = 8 becomes unstable for PPO). Also note that algorithms with mechanisms to control off-policy updates such as PPO with importance-weight clipping are suited to perform better in the off-policy sample reuse setting.}}
    \vspace{-0.2cm}
\end{figure}

\end{document}